\documentclass[a4paper,11pt]{article}
\usepackage[top=4cm, bottom=4cm, left=3cm, right=3cm]{geometry}

\usepackage[utf8]{inputenc} % allow utf-8 input
\usepackage[T1]{fontenc}    % use 8-bit T1 fonts
\usepackage{url}            % simple URL typesetting
\usepackage{amsfonts}       % blackboard math symbols
\usepackage{nicefrac}       % compact symbols for 1/2, etc.
\usepackage{microtype}      % microtypography

\usepackage{amssymb}
\usepackage{amsmath}
\usepackage{amsthm}

\usepackage{thmtools, thm-restate}
\usepackage{mathrsfs} %pour mathscr
\usepackage[usenames,dvipsnames]{pstricks}
\usepackage{euler}
\usepackage{euscript}
\usepackage{graphicx}
\usepackage{charter}
\usepackage{mdframed}
\usepackage{enumerate, subfigure, color}
\usepackage[colorlinks,hyperindex]{hyperref}
\usepackage{bibunits}

\usepackage{newfloat}
\usepackage{caption}

\usepackage{booktabs}
\usepackage{cite}

\usepackage{floatrow}

\usepackage{enumitem}
\usepackage{multirow}

\usepackage{algorithm}
\usepackage{algorithmicx}

\newcommand{\Mb}{{\bf M}}
\newcommand{\Ob}{{\bf O}}
\newcommand{\Mnm}{\Mb_{n,m}}

\newcommand{\hW}{{W_*}}
\newcommand{\hA}{{A_*}}
\newcommand{\hB}{{B_*}}

\newcommand{\loss}{{\ell}}

\newcommand{\prox}{{\rm prox}}
\newcommand{\N}{{\mathbb N}}

\newcommand{\R}{{\mathbb R}}

\newcommand{\diag}{{\rm diag}}

\newcommand{\beq}{\begin{equation}}
\newcommand{\eeq}{\end{equation}} 
\newcommand{\bea}{\begin{eqnarray}}
\newcommand{\eea}{\end{eqnarray}}

\newcommand{\trace}{{\rm tr}}

 \newcommand{\lb}{{\langle}}
\newcommand{\rb}{{\rangle}} 

\newcommand{\SSS}{\scriptscriptstyle}
\def\boldf#1{\hbox{\rlap{$#1$}\kern.4pt{$#1$}}}

 \newcommand{\trans}{^{\scriptscriptstyle
\top}}

\newcommand{\rank}{{\rm rank}}

\newcommand{\minimize}[1]{\underset{#1}{\rm minimize}}
\newcommand{\argmin}[1]{\underset{#1}{\rm argmin}}
\renewcommand{\eqref}[1]{Eq.~(\ref{#1})}
\newcommand{\grad}{\nabla}

\newcommand{\fla}{{f_{\lambda}}}
\newcommand{\glar}{{g_{\lambda,r}}}
\newcommand{\glarplus}{{g_{\lambda,r+1}}}
\newcommand{\glarstar}{{g_{\lambda,r_*}}}

\newcommand{\secref}[1]{Sec.~\ref{#1}}
\renewcommand{\algref}[1]{Alg.~\ref{#1}}

\let\inf\relax \DeclareMathOperator*\inf{\vphantom{p}inf}

\newcommand{\bproof}{\begin{proof}}
\newcommand{\eproof}{\end{proof}}
\makeatletter
\renewenvironment{proof}[1][\proofname]{\par
  \pushQED{\qed}%
  \normalfont \topsep6\p@\@plus6\p@\relax
  \trivlist
  \item[\hskip\labelsep
        \bfseries
    #1\@addpunct{.}]\ignorespaces
}{%
  \popQED\endtrivlist\@endpefalse
}
\makeatother
\def\eop{$\rule{1.3ex}{1.3ex}$}
\renewcommand\qedsymbol\eop

\declaretheorem[name=Theorem,refname=Thm.]{theorem}
\declaretheorem[name=Lemma,sibling=theorem]{lemma}
\declaretheorem[name=Proposition,refname=Prop.,sibling=theorem]{proposition}

\declaretheorem[name=Corollary,refname=Cor.,sibling=theorem]{corollary}
\declaretheorem[name=Definition,refname=Def.,sibling=theorem]{definition}

\title{\sffamily\huge\bf Reexamining Low Rank Matrix Factorization \\for Trace Norm Regularization}

\author{Carlo Ciliberto $^{1}$ \\ {\small\em c.ciliberto@ucl.ac.uk~~~} \and Dimitrios Stamos $^{1}$ \\ {\small\em d.stamos@cs.ucl.ac.uk~~~} \and Massimiliano Pontil $^{1,2}$ \\ {\small\em m.pontil@cs.ucl.ac.uk~~~~~} }

\begin{document}
% \nipsfinalcopy is no longer used

\maketitle

\begin{abstract}
\footnotetext[1]{University College London, London, UK}\footnotetext[2]{Computational Statistics and Machine Learning - Istituto Italiano di Tecnologia, Genova, Italy}Trace norm regularization is a widely used approach for learning low rank matrices. 
A standard optimization strategy is based on formulating the problem as one of low rank matrix factorization which, however, leads to a non-convex problem.
%CARLO - Ho rimosso: solved by alternate minimization. 
%A standard technique to solve this problem is based on a low rank matrix factorization reformulation, which however leads to a non-convex problem. 
In practice this approach works well, and it is often computationally faster than standard convex solvers such as proximal gradient methods. Nevertheless, it is not guaranteed to converge to a global optimum, and the optimization can be trapped at poor stationary points. In this paper we show that it is possible to characterize all critical points of the non-convex problem. This allows us to provide an efficient criterion to determine whether a critical point is also a global minimizer.  
Our analysis suggests an iterative meta-algorithm that dynamically expands the parameter space and allows the optimization to escape any non-global critical point, thereby converging to a global minimizer. The algorithm can be applied to problems such as matrix completion or multitask learning, and our analysis holds for any random initialization of the factor matrices. Finally, we confirm the good performance of the algorithm on synthetic and real datasets.
\end{abstract}

\section{Introduction}

Learning low rank matrices is a problem of broad interest in machine learning and statistics, with applications ranging from collaborative filtering \cite{rennie2005fast,koren2009matrix}, to multitask learning \cite{Argyriou2008}, to computer vision \cite{harchaoui2012large}, and many more. A principled approach to tackle this problem is via suitable convex relaxations. 
Perhaps the most successful strategy in this sense is provided by trace (or nuclear) norm regularization  \cite{Amit2007,Argyriou2008,Bach2008,Srebro2005}. However, solving the corresponding optimization problem is computationally expensive for two main reasons. First, many commonly used algorithms require the computation of the proximity operator (e.g. \cite{Bauschke2010}) which entails performing the singular value decomposition at each iteration. Second, and most importantly, the space complexity of convex solvers grows with the matrix size, which makes it prohibitive to employ them for large scale applications. %such as Netflix.
%Conditional gradient methods \cite{dudik2012lifted} mitigate the first problem, but the space complexity remains.

Due to the above shortcomings, practical algorithms for low rank matrix completion often use an explicit low rank matrix factorization to reduce the number of variables (see e.g. \cite{koren2009matrix,hastie2015matrix} and references therein). In particular, a reduced variational form of the trace norm is used \cite{Srebro2005}. The resulting problem is however non-convex, and popular methods such as alternate minimization or alternate gradient descent may get struck at poor stationary points. Recent studies \cite{NIPS2016_6048,bhojanapalli2016global} have shown that under certain conditions on the data generation process (underling low rank model, RIP, etc.) a particular version of the non-convex problem can be solved efficiently.
%CARLO - I removed: by polynomial algorithms
However such conditions are not verifiable in real applications, and the problem of finding a global solution remains open. 
%Recently, methods based on approximate SVD have been proposed \cite{hsieh2014} to speed up the nuclear norm optimization problem, matching the performance on non-convex solvers.

In this paper we characterize the critical points of the non-convex problem and provide an efficient criterion to determine whether a critical point is also a global minimizer. Our analysis is constructive and suggests an iterative meta-algorithm that dynamically expands the parameter space to escape any non-global critical point, thereby converging to a global minimizer. We highlight the potential of the proposed meta-algorithm, by comparing its computational and statistical performance to two state of the art methods \cite{hsieh2014,hastie2015matrix}.
%CARLO - I removed:  two state of the art methods \cite{hsieh2014,hastie2015matrix} which speed up the trace norm optimization problem via approximate SVD computation.

The paper is organized as follows. In \secref{sec:2} we introduce the trace norm regularization problem and basic notions used throughout the paper. In \secref{sec:3} we present the low rank matrix factorization approach and set the main questions addressed in the paper. In \secref{sec:analysis} we present our analysis of the critical points of the low rank matrix factorization problem and the associated 
%MASSI induced?
meta-algorithm. In \secref{sec:exps} we report results from numerical experiments using our method.~The Appendix contains proofs of the results, only stated in the main body of the paper, together with further auxiliary results and empirical observations.

%\subsection{Notation} We denote vectors by lower case letters and matrices by capital letters. We let $W \in \R^{m\times n}$ and assume without loss of generality assume that $m \geq n$. We let $\|\cdot\|_p$ the $\ell_p$ norm of a vector or the $\ell_p$-Schatten norm of a matrix. The latter is defined as $\|W\|_p = \|\sigma(W)\|_p$. 

\section{Background and Problem Setting}
%MASSI check notation is defined
\label{sec:2}
In this work we study trace norm regularized problems of the form 
\begin{equation}\label{eq:original}
    \minimize{W\in\R^{n \times m}} ~ \fla(W), \qquad \fla(W) =  \ell(W) + \lambda ~ \|W\|_*
\end{equation}
where $\ell:\R^{n \times m}\to\R$ is a twice differentiable convex function with Lipschitz continuous gradient, 
%MASSI: i removed this (we say it anyway below)  with constant $L>0$, 
%with Lipschitz continuous gradient with constant $L>0$, 
$\|W\|_*$ denotes the trace norm of a matrix $W\in\R^{n \times m}$, namely the sum of the singular values of $W$, and 
$\lambda$ is a positive parameter. Examples of relevant learning problems that can be formulated as \eqref{eq:original} are:

\noindent{\it Matrix Completion}. In this setting we wish to recover a matrix $Y\in\R^{n \times m}$ 
from a small subset of its entries. A typical choice for $\ell$ is the square error, $\ell(W) = \|M \odot (Y - W)\|_F^2$, where $\|\cdot\|_F$ is the Frobenius norm, $\odot$ denotes the Hadamard product (i.e. the entry-wise product) between two matrices and $M\in\R^{n \times m}$ is a binary matrix used to ``mask'' the entries of $Y$ that are not available.

\noindent{\it Multi-task Learning}. 
%Trace-norm regularization has been extensively studied in multi-task learning settings \cite{Amit2007,Argyriou2008,Bach2008,Srebro2005}. 
Here, the columns $w_1,\dots,w_m$ of matrix $W$ are interpreted as the regression vectors of different learning tasks. 
%the matrix $W$ at \eqref{eq:original} is interpreted as a vector-valued linear predictor, mapping input points $x\in\R^n$ to outputs $y = W\trans x \in\R^m$. 
Given $m$ datasets $(x_{j}^i,y_{j}^i)_{i=1}^{n_j}$ with $x_{j}^i \in \R^n$ and $y_{j}^i \in \R$, for $j=1,\dots,m$, we choose  
$\ell(W) = \sum_{j=1}^m \frac{1}{~n_j}\sum_{i=1}^{n_j} {\bar \loss}(y_{j}^i,w_j\trans x_{j}^i)$, 
where ${\bar \loss}:\R\times\R\to\R$ is a prescribed loss function (e.g. the square or the logistic loss). 
%Note that if $\loss$ is the square loss and $x_{s,i} \in \{e_1,\dots,e_n\}$ we recover the matrix completion example above.

Other examples which are captured by problem (\ref{eq:original}) include collaborative filtering with attributes \cite{Abernethy2009} and multiclass classification \cite{Amit2007}.
%A more general setting which include is to choose $\ell(W)$ to be the sum of terms of the form $\loss(y_k - \lb W,X_k\rb_F)$, where $\loss$ is a loss function, $y\in \R$, $X \in \R^{n \times m}$ and $\lb \cdot,\cdot\rb_F$ is the Frobenius inner product. 

\subsection{Proximal Gradient Methods}
%\subsection{Convex Solvers}
Problem (\ref{eq:original}) can be solved by first-order optimization methods such as the proximal forward-backward (PFB) splitting \cite{Bauschke2010}. 
%; indeed $f$ is twice differentiable and the proximal operator of $\|\cdot\|_*$ can be computed in closed form, see e.g. \cite{Cai10}. 
Given a starting point $W_0\in\R^{n \times m}$, PFB produces a sequence $(W_k)_{k\in\N}$ with
\begin{equation}\label{eq:pfb}
    W_{k+1} = \prox_{\gamma \lambda \|\cdot\|_*} ~ \big( W_k - \gamma \grad \ell(W_k) \big)
\end{equation}
where $\gamma>0$ and 
%$\prox_{\gamma \lambda \|\cdot\|_*}$ denotes the proximal operator of the nuclear norm. We recall that 
for any convex function $\phi:\R^{n \times m}\to\R$, the associated proximity operator at $W\in\R^{n \times m}$ is defined as $\prox_\phi(W) = \textrm{argmin}\big\{ \phi(Z) + \frac{1}{2}\|W - Z\|_F^2:{Z\in\R^{n \times m}} \big\}$.
In particular, the proximity operator of the trace norm at $W\in\R^{n \times m}$ corresponds to performing a soft-thresholding on the singular values of $W$. That is, assuming a 
%MASSI clarify this is the reduced SVD? also elsewhere in the paper for consistency
{\em singular value decomposition (SVD)} $W = U\Sigma V\trans$, where $U\in\R^{n \times r}$, $V\in\R^{m \times r}$ have orthonormal columns, $\Sigma= \diag(\sigma_1,\dots,\sigma_r)$, with $\sigma_1\geq \cdots \geq \sigma_r > 0$ and $r=\rank(W)$, we have
\begin{equation}
\label{eq:prox-trace}
\prox_{\gamma \lambda \|\cdot\|_*}(W) = U \diag(h_{\gamma\lambda}(\sigma_1),\dots,h_{\gamma\lambda}(\sigma_r)) V\trans 
  \end{equation}
where $h_{\gamma\lambda}$ is the {\em soft-thresholding} operator, defined for $\sigma \geq 0$ as $h_{\gamma\lambda}(\sigma) = \max(0,\sigma-\gamma\lambda)$.
%\cite{bauschke2011}.

%MASSI: should we also mention/discuss soft-impute (Tinshirani etc.)
%
%{\bf Convergence.} 
%Tackling the minimization of $\fla$ with the 
PFB guarantees that for a suitable choice of the descent step $\gamma$ (e.g. $\gamma<2/L$ with $L$ the Lipschitz constant of the gradient of $\ell$), the sequence $\fla(W_k)$ converges to the global minimum of the problem with a rate of $O(1/k)$ \cite{Bauschke2010} (faster rates can be achieved using accelerated versions of the algorithm \cite{Beck2009}). However, at each iteration PFB performs the SVD of an $n \times m$ matrix via Eqs. (\ref{eq:pfb}) and (\ref{eq:prox-trace}), which requires $O(\min(n,m)~nm)$ operations, a procedure that becomes prohibitively expensive for large values of $n$ and $m$. 
Other methods such as those based on Frank-Wolfe procedure (e.g. \cite{dudik2012lifted}) alleviate this cost but require more iterations. More importantly, these methods 
% sequence produced by PFB 
need to store in memory the iterates $W_k$ imposing an $O(nm)$ space complexity, a major bottleneck for large scale applications. However, trace norm regularization is typically applied to problems where the solution of problem (\ref{eq:original}) is assumed to be low-rank, namely of the form $AB\trans$ for some $A\in\R^{n \times r}, B\in\R^{m \times r}$ and $r \ll \min(n,m)$. Therefore it would be ideal to have an optimization method capable to capture this aspect and consequently reduce the memory requirement to $O(r(m + n))$ by keeping track of the two factor matrices $A$ and $B$ throughout the optimization rather than their product. This is the idea behind factorization-based methods, which have been observed to lead to remarkable performance in practice and are the subject of our investigation in this work.
\subsection{Matrix Factorization Approach}
\label{sec:3}
Factorization methods build on the so-called {\em variational form} of the trace norm. Specifically, the trace norm of a matrix $W\in\R^{n \times m}$ can be characterized as (see e.g. \cite{Jameson1987} or \autoref{lemma:trace_norm} in the Appendix)
\beq
    \|W\|_* =   \frac{1}{2} \inf \Big\{\|A\|_F^2 + \|B\|_F^2,~: ~ r\in\N,~ A\in\R^{n \times r},~ B\in\R^{m \times r},~ W = AB\trans\Big\}.
    \label{eq:variational-form}
\eeq
with the infimum always attained for $r=\rank(W)$. The above formulation leads
%suggests to consider 
to the following ``factorized'' version of the original optimization problem (\ref{eq:original})
\begin{equation}\label{eq:factorized-problem}
    \minimize{A\in\R^{n \times r}, B\in\R^{m \times r}} ~ \glar(A,B), \qquad \glar(A,B) =  \ell(AB\trans) + \frac{\lambda}{2} ~ \left(\|A\|_F^2 + \|B\|_F^2\right)
\end{equation}
where $r\in\N$ is now a further hyperparameter of the problem. Clearly, $\fla$ and $\glar$ are tightly related and a natural question is whether minimizing the latter would allow to recover a solution of the original problem. The following well-known result (of which we provide a proof in the Appendix for completeness) shows that the two problems are indeed equivalent (for sufficiently large $r$).

\begin{restatable}[Equivalence between problems (\ref{eq:original}) and (\ref{eq:factorized-problem})]{proposition}{PEquivalence}\label{prop:equivalence}
Let $W_* \in\R^{n \times m}$ be a global minimizer of $\fla$ in \eqref{eq:original} with $r_* = \textrm{rank}(W_*)$. Then, for every $r\geq r_*$, every global minimizer $(A_*,B_*)$ of $\glar$ is such that 
\begin{equation}
    \glar(A_*,B_*) = \fla(A_*B_*\trans) = \fla(W_*).
\end{equation}
\end{restatable}
The above proposition implies that for sufficiently large values of $r$ in \eqref{eq:factorized-problem}, the optimization of $\fla$ and $\glar$ are {\em equivalent}. Therefore, {\em we can minimize $\fla$ by finding a global minimizer for $\glar$}. This is a well-known approach to trace norm regularization (see e.g. \cite{koren2009matrix,rennie2005fast}) and can be extremely advantageous in practice. Indeed, we can leverage on a large body of smooth optimization literature to solve such a factorized problem \cite{Bertsekas1999}. As an example, if we apply Gradient Descent (GD) from a starting point $(A_0,B_0)$, we obtain the sequence $(A_k,B_k)_{k\in\N}$ with
% \begin{align}\label{eq:gd}
% A_{k+1} & = A_k - \gamma \left( \grad \ell(A_k B_k\trans)B_k + \lambda A\right)\\
% B_{k+1} & = B_k - \gamma \left( \grad \ell(A_k B_k\trans)\trans A_k + \lambda B \right).
% \end{align} 
\begin{equation}
\label{eq:gd}
~\left.
\begin{array}{lcl}
A_{k+1}  & = &A_k - \gamma ( \grad \ell(A_k B_k\trans)B_k + \lambda A_k) \\ [2pt]
B_{k+1}  &= & B_k - \gamma ( \grad \ell(A_k B_k\trans)\trans A_k + \lambda B_k ).
\end{array} \right.
\end{equation}
%\begin{equation}\label{eq:gd}
%A_{k+1}  = A_k - \gamma \left( \grad \ell(A_k B_k\trans)B_k + \lambda A\right) ~~ \textrm{and} ~~ B_{k+1} = B_k - \gamma \left( \grad \ell(A_k B_k\trans)\trans A_k + \lambda B \right).
%\end{equation} 
This approach is much more appealing than PFB from a computational perspective because: $1$) for small values of $r$, the iterations at \eqref{eq:gd} are extremely fast since they mainly consists of matrix products and therefore require only $O(nmr)$ operations; $2)$ the space complexity of GD is $O(r(n+m))$, which may be remarkably smaller than the $O(nm)$ of PFB for small values of $r$; $3$) Even if for large values of $r$, e.g. $r=\min(n,m)$, every iteration has the same time complexity as PFB, we do not need to perform expensive operations such as the SVD of an $n \times m$ matrix at every iteration. This dramatically reduces computational times in practice.

%{\em Questions}. 
The strategy of minimizing $\glar$ instead of $\fla$ was originally proposed in \cite{Srebro2005} and its empirical advantages have been extensively documented in previous literature, e.g. \cite{koren2009matrix}. However, this approach opens important 
%several 
theoretical and practical questions that have not been addressed by previous work:
\begin{itemize}
\item {\bf How to choose r?} By \autoref{prop:equivalence} we know that for suitably large values of $r$ the minimization of $\glar$ and $\fla$ are equivalent. However, a lower bound for such an $r$ cannot be recovered analytically from the functional itself and, so, it is not clear how to choose $r$ in practice.
\item {\bf Global convergence}. The function $\glar$ is not jointly convex in the two variables $A$ and $B$. This opens the question of whether GD (or other optimization methods) converge to a global minimizer  for sufficiently large values of $r$.
%\item {\bf Convergence rates}. A final question concerns the rates to which we can guarantee that an optimization method minimizing $\glar$ will converge to a critical point of the functional.
\end{itemize}
Investigating such issues is the main focus of this work.
% In the rest of this paper we investigate the questions above and provide a novel analysis for the factorized problem $\glar$.

\section{Analysis}
\label{sec:analysis}
In this section we study the questions outlined above and provide a meta-algorithm to minimize the function $\glar$ while incrementally searching for a rank $r$ for which \autoref{prop:equivalence} is verified. 
Our analysis builds upon the following keypoints:
\begin{itemize}
  \item (\autoref{thm:characterization}) We characterize all critical points of $\glar$, namely those points to which iterative optimization methods applied to \eqref{eq:factorized-problem} (e.g. GD) could in principle converge to.
  \item (\autoref{thm:criterion}) We derive an efficient criterion to determine whether a critical point of $\glar$ is a global minimizer (typically an NP hard problem for non-convex functions). 
  \item  (\autoref{thm:escape}) We show that for any critical point $(A,B)$ of $\glar$ which is not a global minimizer, it is always possible to {\em constructively} find a descent direction for $g_{\lambda,r+1}$ from the point $([A~ 0], [B~ 0]) \in \R^{n \times  (r+1)} \times \R^{m \times  (r+1)}$.
  \item (\autoref{thm:global-convergence}) By combining the above results we show that for $r \geq \min(n,m)$, every critical point of $\glar$ is either a global minimizer or a so-called {\em strict saddle point}, 
%MASSI: i removed this citation: \cite{lee2016}
namely a point where the Hessian of the target function has at least a negative direction. We can then appeal to \cite{lee2016} to show that descent methods such as GD avoid strict saddle points and hence convergence to a global minimizer.
\end{itemize}
The above discussion suggests a natural ``meta-algorithm'' (which is presented more formally in \secref{sec:meta-algorithm}) to address the minimization of $\fla$ via $\glar$ while increasing $r$ incrementally: 
%MASSI: think if here and the also we can remove the "prime"
\begin{enumerate}
    \item Initialize $r=1$. Choose $A_0'\in\R^n$ and $B_0'\in\R^m$.
    \item Starting from $(A_{r-1}',B_{r-1}')$, converge to a critical point $(A_r,B_r)$ for $\glar$.% (e.g. by GD).
    \item If $(A_r,B_r)$ satisfies our criterion for global optimality (see \autoref{thm:criterion}) stop, otherwise: 
    \item Perform a step in a descent direction for $g_{\lambda,r+1}$ from $([A_r~ 0], [B_r~ 0])$ to a point $(A'_r,B'_r)$, $A_r'\in\R^{n \times r+1}, B_r'\in\R^{m \times r+1}$; Increase $r$ to $r+1$ and go back to Step $2$.
\end{enumerate}
From our analysis in the following, the procedure above is guaranteed to stop {\em at most} after $r = \min(n,m)$ iterations. However, \autoref{prop:equivalence}, together with our criterion for global optimality, suggests that this meta-algorithm could stop much earlier if $\fla$ admits a low-rank minimizer (which is indeed the case in our experiments). This has two main advantages: $1$) by exploring candidate ranks incrementally, we can expect significantly faster computations and convergence if our optimality criterion activates for $r \ll \min(n,m)$ and 2) we automatically recover the rank of a minimizer for $\fla$ without the need to perform expensive operations such as SVD.

\begin{restatable}{remark}{RVidal}\label{rem:vidal} The meta-algorithm considered in this paper is related to the optimization strategy recently proposed in \cite{Vidal}, where the authors study convex problems for which a non-convex ``factorized'' formulation exists, including the setting considered in this work as a special case. However, by adopting such a general perspective, the resulting optimization strategy is less effective when applied to the specific minimization of $\glar$. In particular: $1)$ the optimality criterion derived in \cite{Vidal} is only a sufficient but not necessary condition;
%, which turns out to be rarely verified in practice (see Appendix);
%MASSI: what do we say in the appendix?
$2)$ the upper bound on $r$ is much larger than the one provided in this work, i.e. $r = nm$ rather than $r = \min(n,m)$; $3)$ convergence guarantees to a global optimum cannot be provided. 

By focusing exclusively on the question of minimizing $\fla$ via its factorized form $\glar$ and, by leveraging on the specific structure of the problem, it is instead possible to provide a further analysis of the behavior of the proposed meta-algorithm.
\end{restatable}

\subsection{Critical Points %of $\glar$ 
and a Criterion for Global Optimality}
\label{sec:characterization-criterion}

Since $\glar$ is a non-convex smooth function, in principle we can expect optimization algorithms based on first or second order methods to converge only to critical points, i.e. points $(A_*,B_*)$ for which $\grad \glar(A_*,B_*) = 0$. The following result provides a characterization of such critical points and plays a key role in our analysis. 
\begin{restatable}[Characterization of Critical Points of $\glar$]{proposition}{PCharacterization}\label{thm:characterization}
Let $(A_*,B_*)\in\R^{n\times r}\times\R^{m \times r}$ be a critical point of $\glar$. Let $s\leq \min(n,m)$ and let $U\in\R^{n \times s}$ and $V\in\R^{m \times s}$ be two matrices with orthonormal columns corresponding to the left and right singular vectors of $\grad \ell(A_* B_*\trans)\in\R^{n \times m}$ with singular value equal to $\lambda$. Then, there exists $C\in\R^{s \times r}$, such that $A_* = U C$ and $B_* = - V C$.
\end{restatable} 
%
%The above observation is particularly relevant to derive a necessary and sufficient condition to determine whether a
This result is instrumental in deriving a necessary and sufficient condition to determine whether a  stationary point for $\glar$ is actually a global minimizer. Indeed, since $\fla$ is convex, we can leverage on the first order condition for global optimality, stating that the matrix $W_* = A_*B_*\trans$ is a global minimizer for $\fla$ (and by \autoref{prop:equivalence} also $(A_*,B_*)$ for $\glar$) if and only if the zero matrix belongs to its subdifferential  (see e.g. \cite{Bauschke2010}). Studying this inclusion leads to the following theorem.
\iffalse
\begin{equation}
  0 \in \partial\fla(W_*) = \grad \ell(W_*) + \lambda\partial \|W_*\|_* \end{equation}
where, for every $W \in \R^{n \times m}$, $\partial \|W\cdot\|_*$ denotes the sub-differential of the trace norm at $W$. Furthermore, if $W$ has rank $r$ and singular value decomposition $W = U\Sigma V\trans$, we have (see \cite{watson1992})
\begin{equation}
\partial \|W\|_* = \{UV + N \ | \ N\in\R^{n \times m}, U\trans N = 0, N V\trans = 0, \|N\|\leq 1\}
\end{equation}
where $\|N\|$ denotes the operator norm of $N$, i.e. its largest singular value. Now recall by \autoref{thm:characterization} that $A_* = UC$, $B_* = - VC$ and $\grad \ell(A_*B_*\trans) = \lambda UV\trans + \lambda N'$, with $U\trans N' = 0, N' V\trans = 0$.  Therefore, any element of $\partial\fla(A_*B_*\trans)$ will be of the form $\grad \ell(A_*B_*\trans) + \lambda H = N' - \lambda N$ with $H = UV + N\in\partial\|A_*B_*\trans\|_*$. We can conclude that the matrix $0\in\R^{n \times m}$ belongs to the sub-differential of $\fla$ if and only if $\|N'\|\leq\lambda$ (and consequently $\|\grad \ell(A_*B_*\trans)\|_*\leq\lambda$). This is exactly the necessary and sufficient condition we need to verify in order to reach a global minimizer. The following theorem formalizes this reasoning.
%
\fi
\begin{restatable}[A Criterion for Global Optimality]{theorem}{TCriterion}\label{thm:criterion}
Let $(A_*,B_*)$ be a critical point of $\glar$. Then $A_*B_*\trans$ is a minimizer for $\fla$ if and only if $\|\grad \ell(A_* B_*\trans)\|\leq\lambda$.
\end{restatable}
This result provides a natural strategy to determine whether a descent method minimizing $\glar$ has converged to a global minimizer, that is we evaluate the operator norm of the gradient, denoted $\|\grad \ell(A_*B_*\trans)\|$, and then check whether it is larger than $\lambda$. For large matrices this operation can be performed efficiently by using approximation methods, e.g. power iteration \cite{woodruff2014}. Note that in general it is an NP-hard problem to determine whether a critical point of a non-convex function is actually a  global minimizer \cite{murty1987}; it is only because of the relation with the convex function $\fla$ that in this case it is possible to perform such check in polynomial time.

\subsection{Escape Directions and Global Convergence}
\label{sec:escape}

In this section, we observe that for any critical point of $\glar$ which is not a global minimizer, it is always possible to either find a direction to escape from it or alternatively to increase $r$ by one and to find a decreasing direction for $\glarplus$. This strategy is suggested by the following result.
\begin{restatable}[Escape Direction from Critical Points]{proposition}{PEscape}\label{thm:escape}
With the same notation of \autoref{thm:characterization}, assume $\textrm{rank}(A_*) = \textrm{rank}(B_*) < r$ and $\|\grad \ell(A_*B_*\trans)\| = \mu >\lambda$. Then, $(A_*,B_*)$ is a so-called strict saddle point for $\glar$, namely the Hessian of $\glar$ at~$(A_*,B_*)$~has at least one negative eigenvalue. In particular, there exists $q\in\R^r$ such that $A_*q = 0$, $B_*q = 0$ and if $u\in\R^n$ and $v \in\R^m$ are the left and right singular vectors of~$\grad \ell(A_*B_*\trans)$~with 
%corresponding 
singular value equal to $\mu$, then $\glar$ decreases locally at $(A_*,B_*)$ along the direction $(uq^\top,-vq^\top)$.
\end{restatable}
A direct consequence of the result above is that an optimization strategy can remain trapped only at global minimizers of $\glar$ or at critical points $(A_*,B_*)$ for which $A_*$ and $B_*$ have full rank (since we can always escape from rank deficient ones). If the latter happens, \autoref{thm:escape} suggests the strategy adopted in this work, namely to increase the problem dimension to $r+1$ and consider the ``inflated'' point $([A_*~ 0], [B_*~ 0])$. %\in \R^{n \times (r+1)} \times \R^{m \times (r+1)}$ .
Indeed, at such point, $\glarplus$ attains the same value as $\glar(A_*B_*^\top)$ and it is straightforward to verify that $([A_*~0],[B_*~0])$ is still a critical point for $\glarplus$. Since matrices $[A_*~0]$ and $[B_*~0]$ have now rank $<r+1$, we can apply \autoref{thm:escape} to find a direction along which $\glarplus$ decreases. This procedure will stop for $r>\min(n,m)$ since $\textrm{rank}(A_*) \leq \min(n,m) < r$ and we can therefore apply \autoref{thm:escape} to always escape from critical points until we reach a global minimizer (this fact can be actually improved to hold also for $r=\min(n,m)$ as we see in the following). 

Note however that in general, if the number of non-global critical points is infinite, it is not guaranteed that such strategy will converge to the global minimizer. However, since by \autoref{thm:escape} every such critical point is a {\em strict} saddle point, we can leverage on previous results from the non-convex optimization literature (see \cite{lee2016} and references therein) in order to prove the following result.
\begin{restatable}[Convergence to Global Minimizers of $\glar$]{theorem}{TConvergence}\label{thm:global-convergence}
Let $r\geq\min(n,m)$. Then the set of starting points $(A_0,B_0)$ for which GD does not converge to a global minimizer of $\glar$ has measure zero.
\end{restatable}
%MASSI: next paragraph is not very clear
%CARLO: I have improved it, see if it is better
In particular, \autoref{thm:global-convergence} suggests to initialize the optimization method used to minimize $\glar$ by applying a small perturbation to the initial point $(A_0,B_0)$ via additive noise according to a distribution that is absolutely continuous with respect to the Lebesgue measure of $\R^{n \times r}\times\R^{m \times r}$ (e.g. a Gaussian). This  perturbation guarantees such initial point to not be in the set of points for which GD converges to strict saddle point and therefore that the meta-algorithm considered in this work converges to a global minimizer. We make this statement more precise in the next section.

\subsection{A Meta-algorithm to Minimize $\fla$}
\label{sec:meta-algorithm}
We can now formally present the meta-algorithm outlined at the beginning of \secref{sec:analysis} to find a solution of the trace norm regularization problem (\ref{eq:original}) by minimizing $\glar$ in \eqref{eq:factorized-problem} for increasing values of $r$. 
%while incrementing the estimated rank $r$.
\begin{algorithm}[t]
   \caption{\textsc{Meta-Algorithm}}
   \vspace{-.031truecm}
   \label{alg:meta-algorithm}
\begin{algorithmic}
   \State ~
   \State {\bfseries Input:} $\lambda>0$, $\epsilon_{\textrm{conv}}>0$ convergence tolerance, $\epsilon_{\textrm{crit}}>0$ global criterion tolerance.
   \State {\bfseries Initialize:} Set $r=1$. Sample $A_0'\in\R^n$ and $B_0'\in\R^m$ randomly.
%\vspace{.071truecm}
   \State {\bfseries For} $r = 1$ to $\min(n,m)$
   \State \quad $(A_r,B_r) =$ {\sc OptimizationAlgorithm}($A_{r-1}',B_{r-1}',\glar,\epsilon_\textrm{conv}$)
   \State \quad {\bfseries If} $\|\grad \ell(A_r B_r\trans)\|\leq\lambda + \epsilon_{\textrm{crit}}$
   \State \quad \quad {\bfseries Break} 
   \State \quad $(A_r',B_r') = ([A_r~ u], [B_r~ v])$ with $u\in\R^n, v\in\R^m$ sampled randomly.
   \State \quad $r = r + 1$
   \State {\bfseries End}
%\vspace{.071truecm}   
   \State {\bfseries Return $(A_r,B_r)$} 
\end{algorithmic}
\vspace{-.031truecm}
\end{algorithm}
Algorithm \ref{alg:meta-algorithm} proceeds by iteratively applying the descent method {\sc OptimizationAlgorithm} (e.g. GD) to minimize $\glar$ while increasing the estimated rank $r$ one step at the time. Whenever the optimization algorithm converges to a critical point $(A_r,B_r)$ of $\glar$ (within a certain tolerance $\epsilon_{\textrm{conv}}$), Algorithm \ref{alg:meta-algorithm} verifies whether the global optimality criterion has been activated (again within a certain tolerance $\epsilon_{\textrm{crit}}$). If this is the case, $(A_r,B_r)$ is a global minimizer and we stop the algorithm. Otherwise we ``inflate'' the two factor matrices by one column each and repeat the procedure. The new column is initialized randomly, since by \autoref{thm:escape} we know that we will not converge again to $([A_r~ 0],[B_r~0])$ because it is not full rank. A more refined approach would be to choose $u$ and $-v$ to be the singular vectors of $\grad \ell(A_r B_r\trans)$ associated to the highest singular value. For the sake of brevity we provide an example of such strategy in the Appendix.
%MASSI do we? also say "appendix" 
%CARLO I have added the discussion in the appendix
However note that we still need to apply a random perturbation to the step in the descent direction in order to invoke \autoref{thm:global-convergence} and be guaranteed to not converge to strict saddle points. As a direct corollary to \autoref{thm:global-convergence} we have
\begin{corollary}
\label{cor:6}
Algorithm \ref{alg:meta-algorithm} with GD as {\sc OptimizationAlgorithm} converges to a point $(A_*,B_*)$ such that $A_*B_*\trans$ is a global minimizer for $\fla$ with probability $1$.
\end{corollary}

\subsection{Convergence Rates}
\label{sec:rates}
In this section, we touch upon 
%give address A final question which we touch upon here concerns the rates to which we can guarantee that an optimization method minimizing $\glar$ will converge to a critical point of the functional. 
the question of convergence rates for optimization schemes applied to problem (\ref{eq:factorized-problem}). For simplicity, we focus on GD, but our result generalizes to other first order methods. We provide upper bounds to the number of iterations required to guarantee that GD iterates are $\epsilon$ close to a critical point of the target function. 
% Our goal is to provide upper bounds to the number of iterations required to guarantee that GD iterates are $\epsilon$ close to a critical point of the target function. 
By standard convex analysis results (see ~\cite{boyd2004}) it is known that GD applied to a differentiable convex function is guaranteed to have sublinear convergence of $O(1/\epsilon)$ comparable to that of PFB. However, since $\glar$ is non-convex, here we need to rely on more recent results that investigated the application of first order methods to functions satisfying the so-called {\em Kurdyka-Lojasiewicz (KL) inequality}~\cite{attouch2010,bolte2010}. 
%This notion has received attention during the last few years in particular in light of results such as \autoref{thm:rates}, providing convergence rates for standard optimization methods applied to non-convex functions. 
\begin{definition}[Kurdyka-Lojasiewicz Inequality]\label{def:kl}
A differentiable function $g:\R^d\to\R$ is said to satisfy the Kurdyka-Lojasiewicz inequality at a critical point $x_*\in\R^d$ if there exists a neighborhood $U \subseteq \R^d$ and constants $\gamma,\epsilon>0$ and $\alpha\in[0,1)$ such that, for all $x \in U \cap \{x :  g(x_*) < g(x) < g(x_*) + \epsilon \}$,
\begin{equation}\label{eq:kl}
    \gamma |g(x) - g(x_*)|^\alpha < \|\grad g(x)\|_F.
\end{equation}
\end{definition}
%MASSI: I do not understand the meaning of alpha=0
The KL inequality is a measure of how large is the gradient of the function in the neighborhood of a critical point. This allows us to derive convergence rates for GD methods that depend on the constant $\alpha$. In particular, as a corollary to \cite{lee2016} (see also \cite{attouch2010}) we obtain the following result.
%The KL inequality has received attention during the last few years in particular in light of results such as \autoref{thm:rates}, providing convergence rates for standard optimization methods applied to non-convex functions. 
%
\begin{corollary}[Convergence Rate of Gradient Descent]\label{thm:rates}
Let $(A_k,B_k)_{k\in\N}$ a sequence produced by GD method applied to $\glar$. If $\glar$ satisfies the KL inequality for some $\alpha \in[0,1)$, then there exists a critical point $(A_*,B_*)$ of $\glar$ and constants $C>0$, $b\in(0,1)$ such that
\begin{equation}
\label{eq:gd-rate}
\|(A_k,B_k) - (A_*,B_*)\|_F^2 \leq \left\{
\begin{array}{ll}
Cb^k & \text{if }  \alpha \in(0,1/2],\\
C k^{-\frac{1-\alpha}{2\alpha - 1}} & \text{if } \alpha\in(1/2,1).
\end{array} \right.
\end{equation}
Furthermore, if $\alpha = 0$ convergence is achieved in a finite number of steps.
\end{corollary}
This result shows that depending on the constant $\alpha\in[0,1)$ appearing in the KL inequality, we can expect different convergence rates for GD applied to problem (\ref{eq:factorized-problem}).
Although, it is a challenging task to identify such constant or even provide an upper bound in specific instances, the class of functions satisfying the KL inequality is extremely large and includes both analytic and {\em semi-algebraic} functions, see e.g. \cite{attouch2009}. In the Appendix, we argue that if a function $\ell:\R^{n \times m}\to\R$ is a semi-algebraic, then also $\ell_r:\R^{n \times r}\times\R^{m\times r}\to\R$ such that $\ell_r(A,B) = \ell(AB\trans)$ is semi-algebraic. Therefore, in order to apply \autoref{thm:rates} it is sufficient that the error $\ell$ and the regularizer are semi-algebraic, a property which is verified by many commonly used error functions (or the associated loss functions, e.g. square, logistic) and regularizers (in particular the squared Frobenius norm).
% $\|\cdot\|_F^2$).

\section{Experiments}
\label{sec:exps}

We report an empirical analysis of the meta-algorithm studied in this work. All experiments were conducted on an Intel Xeon E5-2697 V3 2.60Ghz CPU with 32GB RAM. We consider a matrix completion setting, with the loss
%where the target function is 
$\ell(W) = \|M\odot(Y - W)\|_F^2$ , where $Y$ is the matrix that we aim to recover and $M$ a binary matrix masking entries not available at training time. Below we briefly described the datasets used.

{\bf Synthetic.} We generated a $100\times100$ matrix $Y = AB^\top + E$ as a rank $10$ matrix product of two $100\times10$ 
%MASSI was 15 before
matrices $A,B$ plus additive noise $E$; the entries for $A,B$ and $E$ were independently sampled from a standard normal distribution.

{\bf Movielens.} This dataset~\cite{harper2016} consists of different datasets for movie recommendation of increasing size. They all comprise a number of ratings (from 1 to 5) given by $n$ users on a database of $m$ movies, which are recorded as a $Y \in\R^{n \times m}$ matrix with missing entries. In this work we considered three such datasets of increasing size, namely Movielens $100k$ ($ml100k$) with $100$ thousand ratings from $n = 943$ users on $m= 1682$ movies, Movielens $1m$ ($ml1m$) with $\sim1$ million ratings, $n = 6040$ users and $m = 3900$ movies and Movielens $10m$ ($ml10m$), with $\sim10$ millions ratings, $n = 71567$ users and $10681$ movies.

% \begin{itemize}[leftmargin=2em,topsep=0em] 
%  \setlength\itemsep{0em}
%  \item {\bf Synthetic.} We generated a $100\times100$ matrix $Y = AB^\top + E$ as a rank $10$ matrix product of two $100\times15$ matrices $A,B$ plus additive noise $E$; the entries for $A,B$ and $E$ where independently sampled from a standard normal distribution.
%  \item {\bf Movielens.} Movielens~\cite{harper2016} consists of different datasets for movie recommendation of increasing size. They all comprise a number of ratings (from 1 to 5) given by $n$ users on a database of $m$ movies, which are recorded as a $Y \in\R^{n \times m}$ matrix with missing entries. In this work we considered three such datasets of increasing size, namely Movielens $100K$ ($ml100k$) with $100k$ ratings from $n = 943$ users on $m= 1682$ movies, Movielens $1$ million ($ml1m$) with about $1$ million ratings, $n = 6040$ users and $m = 3900$ movies and Movielens $10$ Millions ($ml10m$), with about $10$ millions ratings, $n = 71567$ users and $10681$ movies.  
%  \end{itemize}

\begin{figure}[!t]
  \centering
  \includegraphics[width=0.35\textwidth]{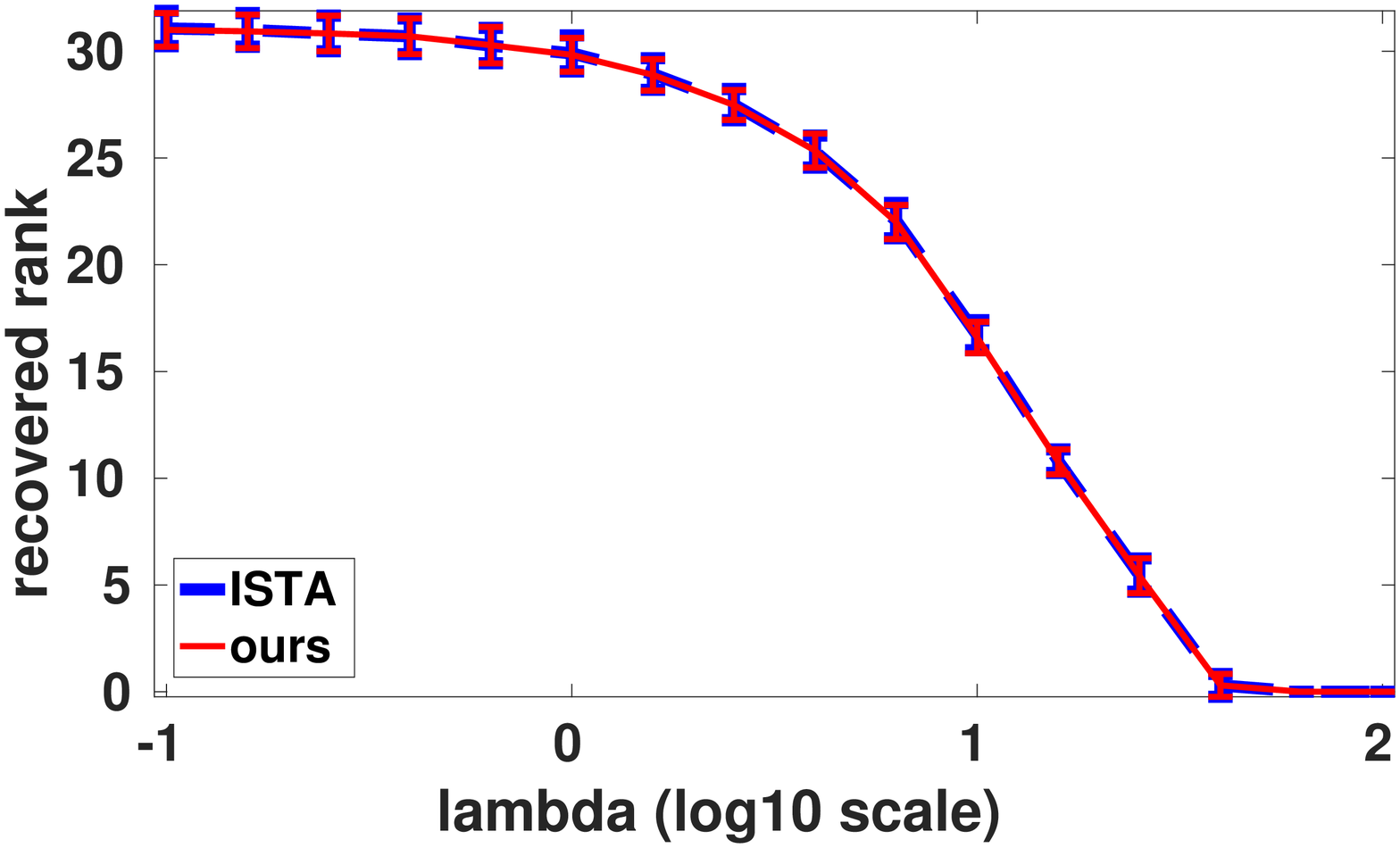}\qquad\quad%
  \includegraphics[width=0.35\textwidth]{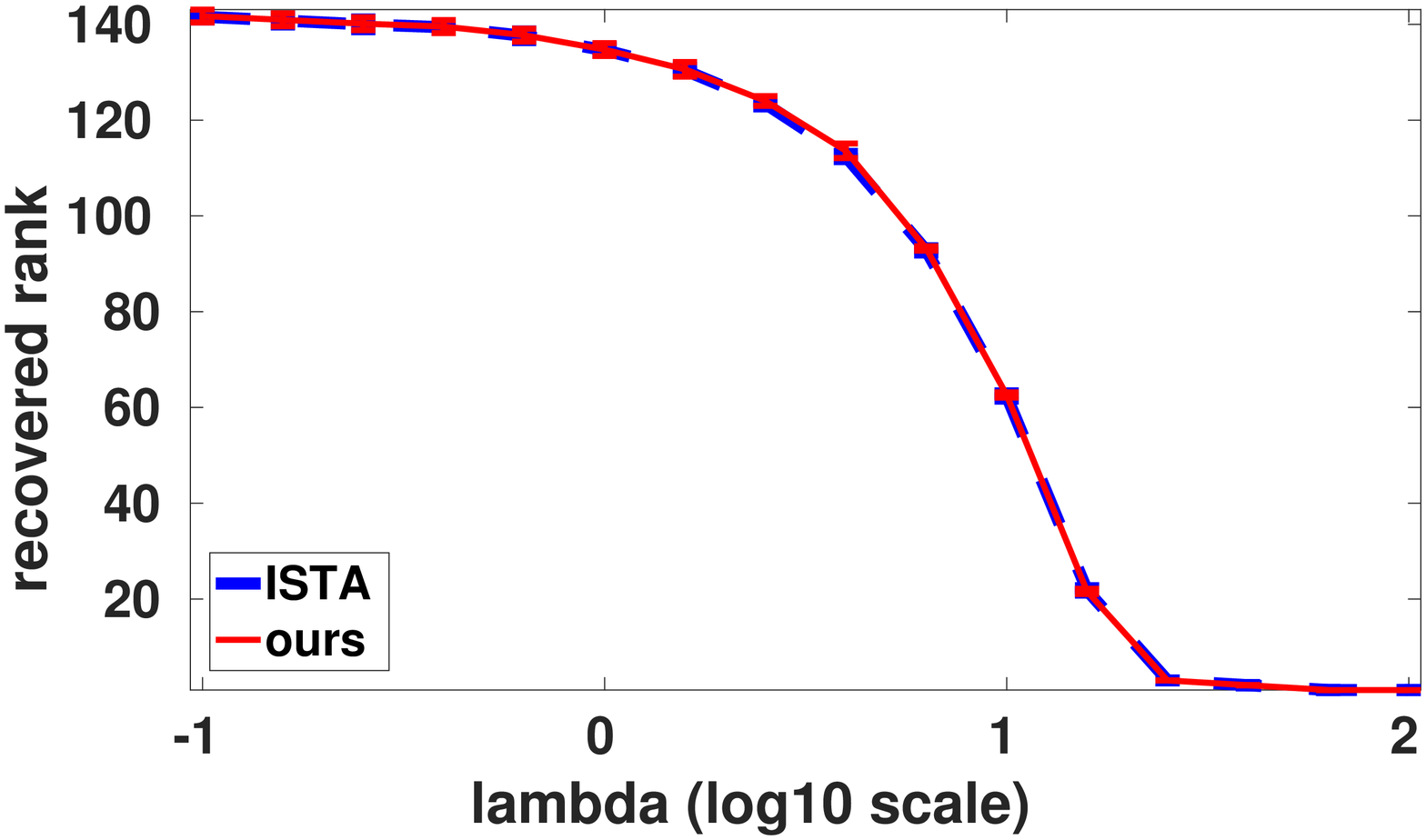}
  \caption{Rank of the solution for iterative soft thresholding algorithm (ISTA) and our method in Algorithm \ref{alg:meta-algorithm} varying lambda, on synthetic data (Left) and ml$100$k (Right). \label{fig:predRank}}
  \vspace{-.3truecm}
\end{figure}

\subsection{The Criterion for Global Optimality and the Estimated Rank}
\label{sec:exp-criterion}
The result in \autoref{thm:criterion} provides a necessary and sufficient criterion to determine whether Algorithm~\ref{alg:meta-algorithm} has achieved a global minimum for $\fla$. A natural question is to ask whether, in practice, such criterion will be satisfied for a much smaller rank $r$ than the one at which we are guaranteed convergence, namely $r = \min(n,m)$. To address this question we compared the solution achieved by our approach with the one obtained by iterative soft thresholding (ISTA)~(or proximal forward backward, see e.g. \cite{Bauschke2010}) on both synthetic and real datasets. Fig.~\ref{fig:predRank} reports the value of $r$ for which our meta-algorithm satisfied the criterion for global optimality and compares it with the rank of the ISTA solution for different values of $\lambda$. For the Synthetic dataset (Fig.~\ref{fig:predRank} Left ) we considered only $20\%$ of the generated matrix $Y$ entries for the optimization of $\fla$. For the real dataset (Fig.~\ref{fig:predRank} Right ) we considered $ml100k$ and sampled $50\%$ of each user's available ratings. For both synthetic and real experiments our meta-algorithm recovers the same rank than as that found by ISTA. However, our algorithm reaches such rank incrementally, exploiting the low rank factorization.
As we will see in the next section this results in a much faster convergence speed in practice.

\subsection{Large Scale Matrix Completion}
%[??? CARLO ???] Commento su comparison con ALS-SI
We compared the performance of our meta-algorithm with two state of the art methods, Active Newton (ALT)~\cite{hsieh2014} and Alternating Least Squares Soft Impute (ASL-SI)~\cite{hastie2015matrix} on the three Movielens datasets. We used $50\%, 25\%$ and $25\%$ of each user's ratings for training, validation and testing and repeated our experiments across $5$ separate trials to account for statistical variability in the sampling. Test error was measured in terms of the Normalized Mean Absolute Error (NMAE), namely the mean (entry-wise) absolute error on the test set, normalized by the maximum discrepancy $\max(Y_{ij}) - \min(Y_{ij})$ between entires in $Y$. As a reference of the behavior of the different methods, Fig.~\ref{fig:convANDerrors}, reports on a single trial the decrease of $\fla$ on training data and NMAE on the test set with respect to time for the best $\lambda$ chosen by validation. All methods where run until convergence and attained the same value of the objective function and same test error in all our trials. However, as it can be noticed, our meta-algorithm and ALS-SI seem to attain a much lower test error during earlier iterations. To better investigate this aspect, Table.~\ref{tab:results} reports results in terms of time, test error and estimated rank attained on average across the $5$ trials by the different methods {\em at the iteration with smallest validation error}.  As it can be noticed our meta-algorithm is generally on par with its competitors in terms of test error while being relatively faster and recovering low-rank solutions. This highlights an interesting byproduct of the meta-algorithm considered in this work, namely that by exploring candidate ranks incrementally, the method allows to find potentially better factorizations than trace norm regularization both in terms of test error and estimated rank. This fact can be empirically observed also for different values of $\lambda$ as we report in the Appendix.

% to confirm they all achieve same minimum value for the target objective function $\fla$ for different values of $\lambda$. Fig.~\ref{fig:convANDerrors} reports both the value of $\fla$ and NMAE with respect to time for $\lambda$ chosen by validation during a single trial. As can be noticed both our meta-algorithm and \cite{hastie2015matrix} achieve much lower test error at intermediate steps and then 

% Table.~\ref{tab:results} reports results in in terms of time, test error and estimated rank.  As it can be noticed our meta-algorithm is generally on par with its competitors in terms of test error while being relatively fast and recovering low-rank solutions. As a further reference to this, Fig.~\ref{fig:convANDerrors} reports both the value of $\fla$ and NMAE with respect to time for $\lambda$ chosen by validation. It can be observed that our method is always on par or better than its competitors, in particular achieving a low test error during the first iterations. This observation highlights an interesting byproduct of the meta-algorithm considered in this work, namely that by exploring candidate ranks incrementally, the method allows to find potentially better factorizations than trace norm regularization both in terms of test error and estimated rank. This fact can be empirically observed also for different values of $\lambda$ as we report in the Appendix. 

\begin{figure}[!t]
    \centering
    \includegraphics[width=0.3\textwidth]{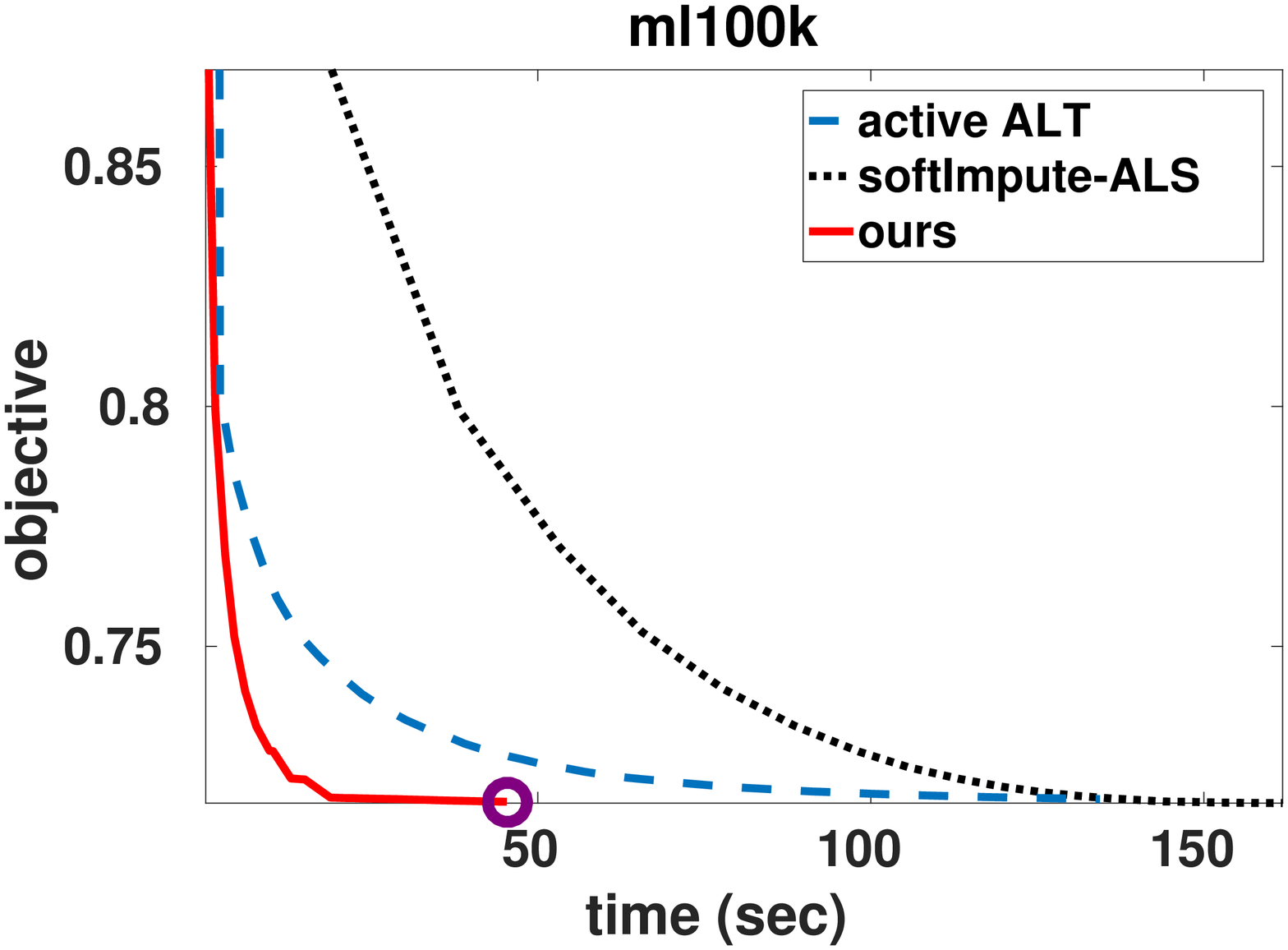}\quad%
    \includegraphics[width=0.3\textwidth]{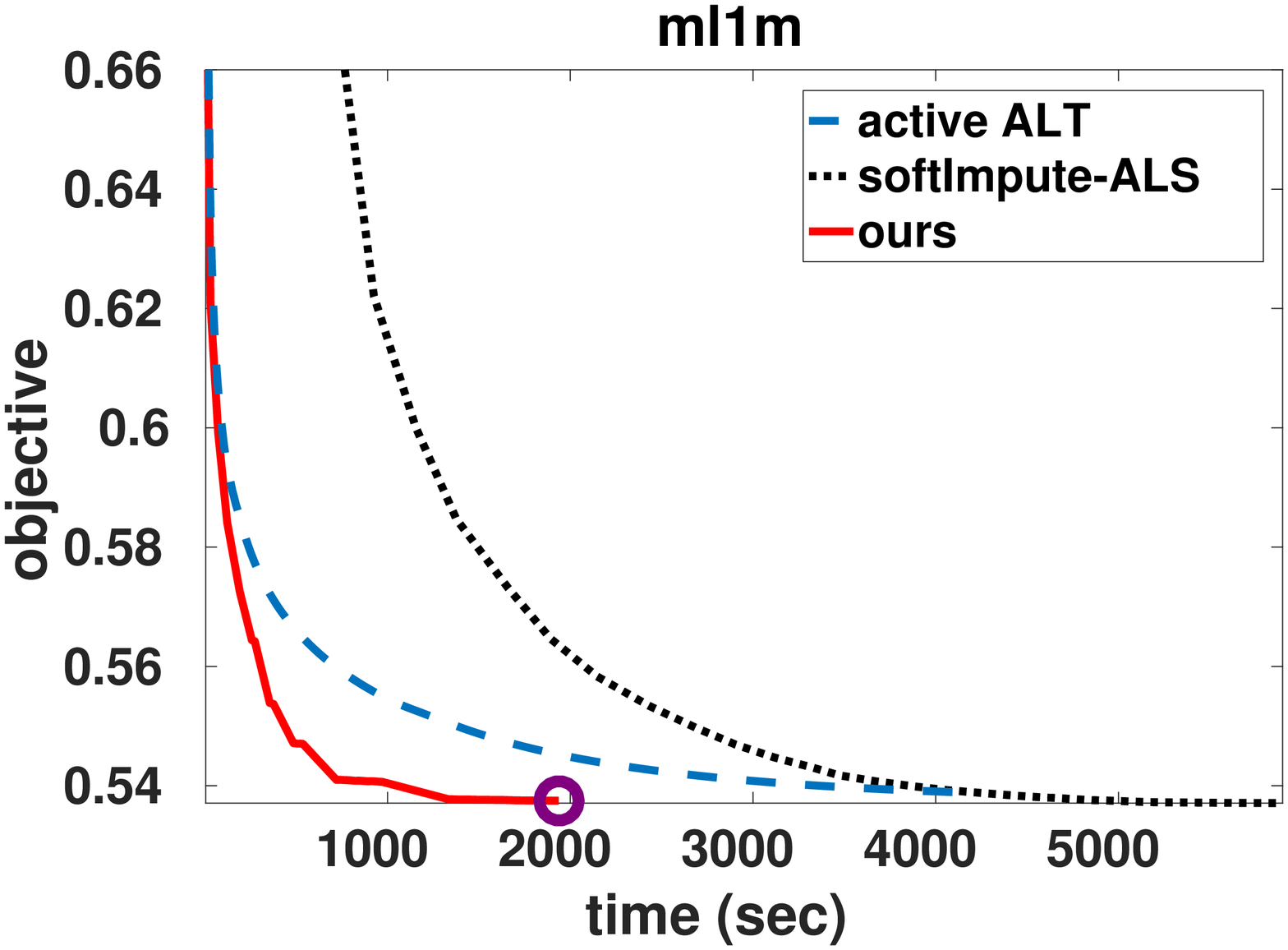}\quad%
    \includegraphics[width=0.3\textwidth]{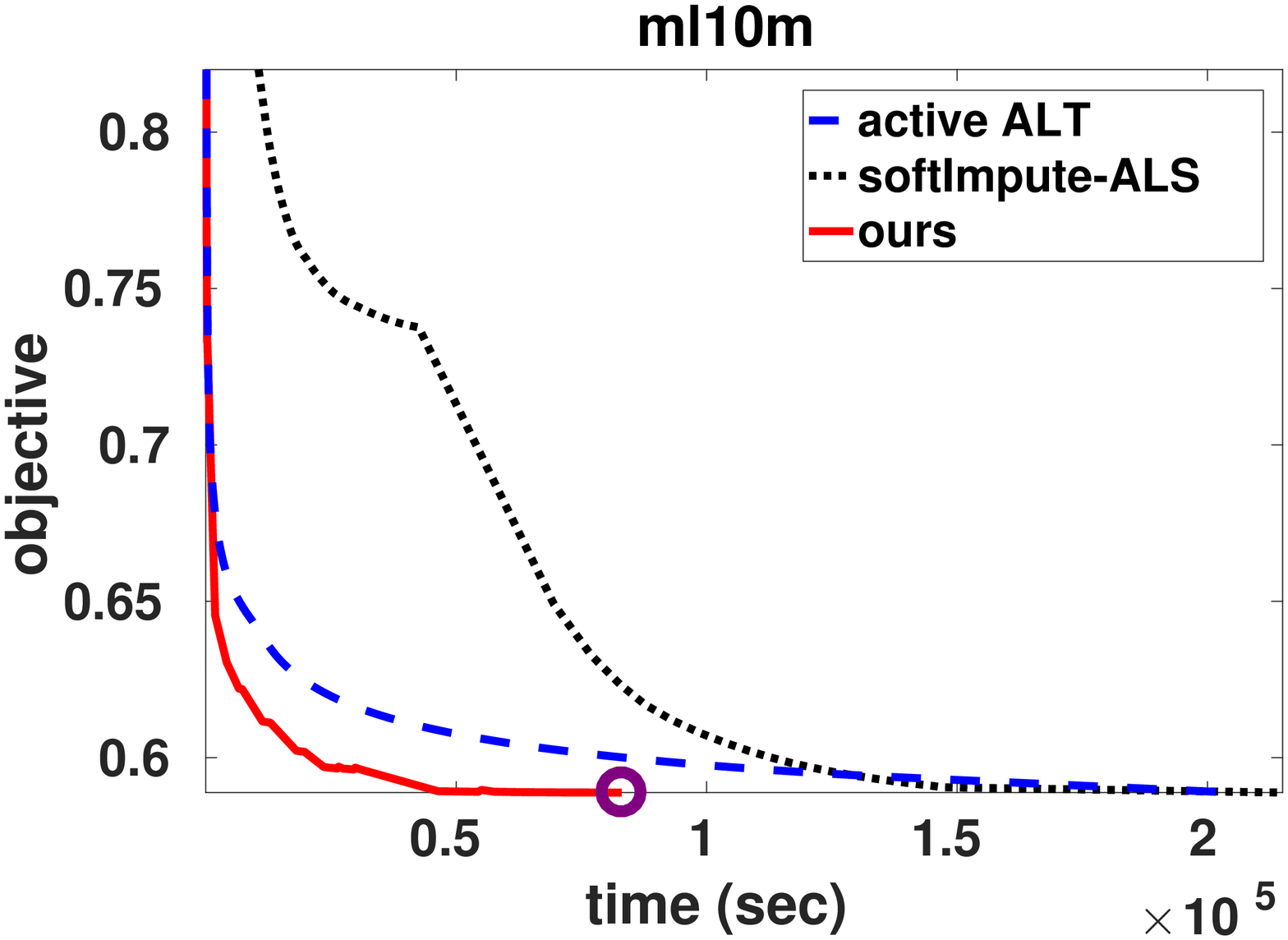}
    ~
    \includegraphics[width=0.3\textwidth]{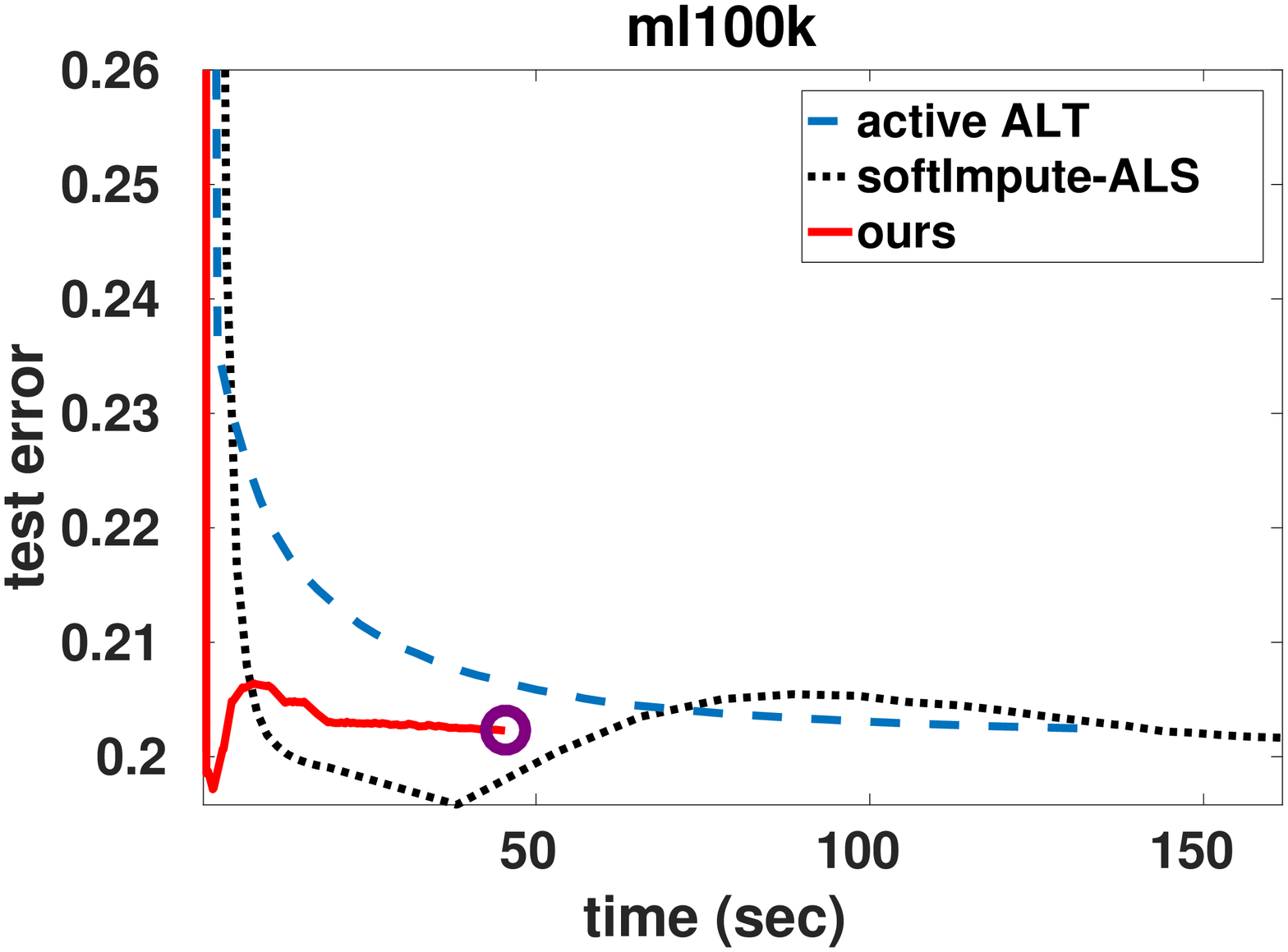}\quad%
    \includegraphics[width=0.3\textwidth]{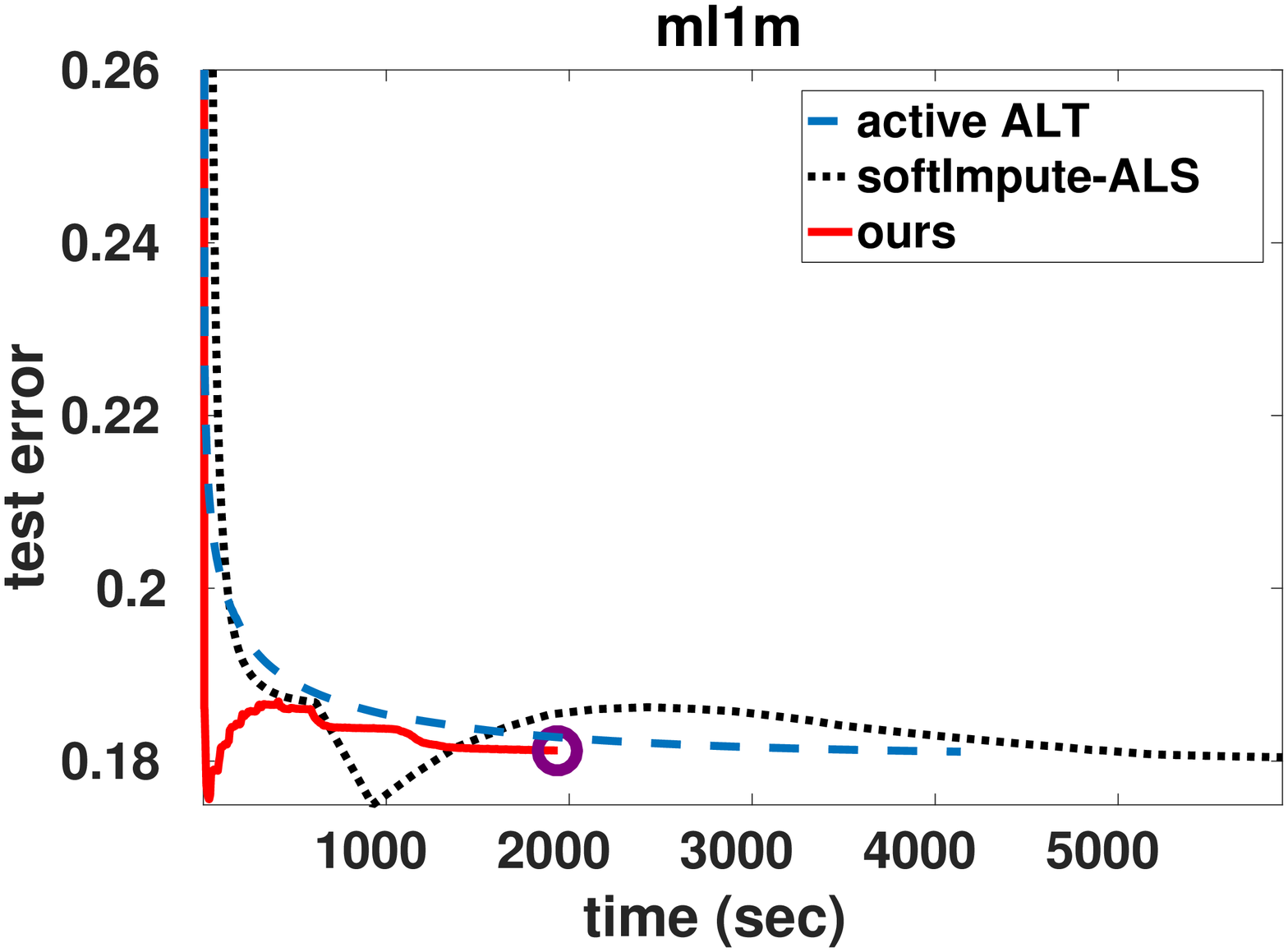}\quad%
    \includegraphics[width=0.3\textwidth]{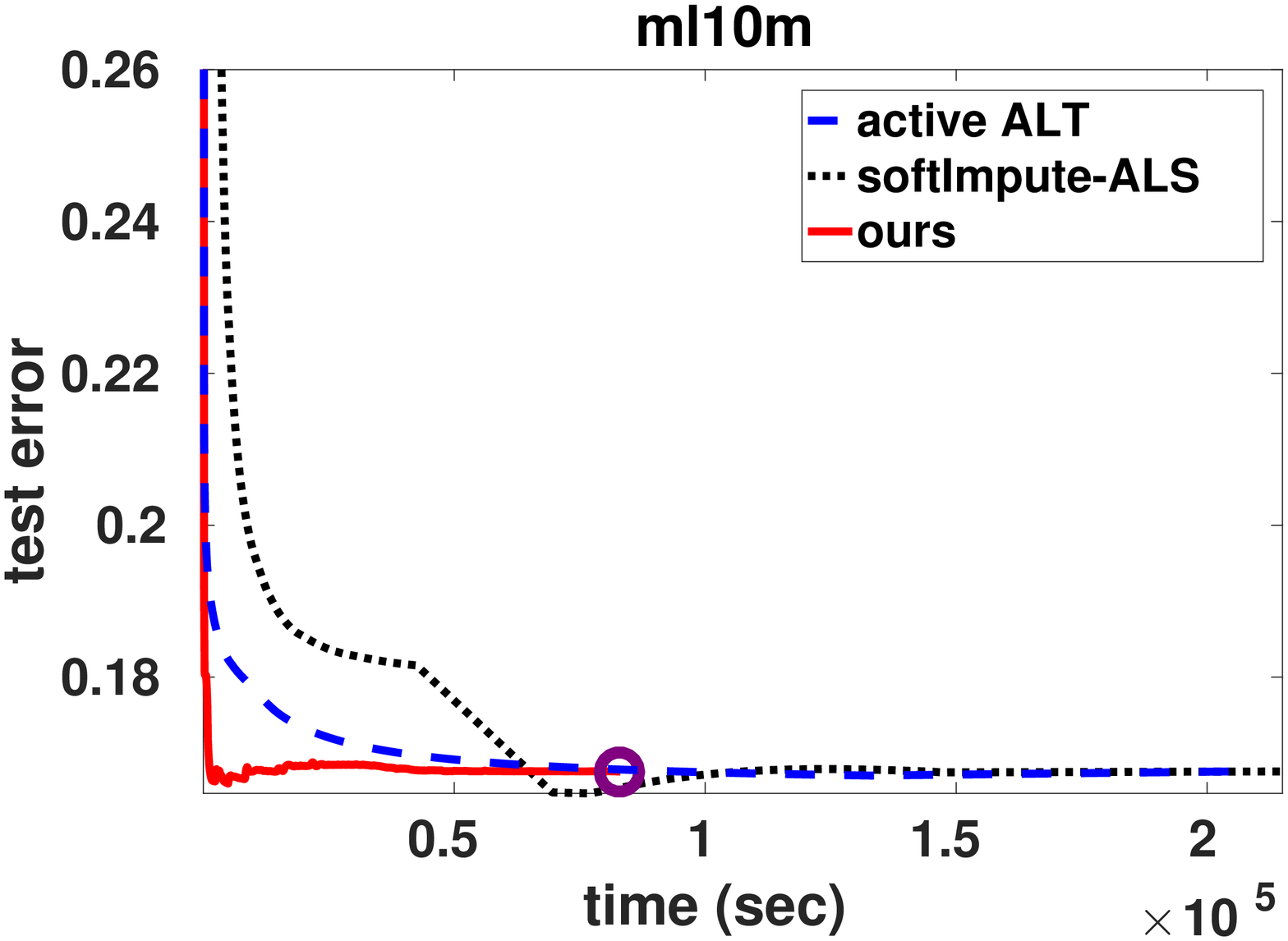}
    \caption{Convergence vs. time of the objective $\fla$ (Top row) and test errors (Bottom row) on three matrix completion large scale datasets for our meta-algorithm, ALT~\cite{hsieh2014} and. The purple circle indicates when the global optimality criterion from \autoref{thm:criterion} is satisfied.}\label{fig:convANDerrors}
\end{figure}

\begin{table}[t]
\small
\caption{Average Normalized Mean Absolute Error (NMAE), convergence time and estimated rank achieved for the best validation parameters by ALT~\cite{hsieh2014}, 
ALS-SI~\cite{hastie2015matrix} 
and \algref{alg:meta-algorithm}, on the three Movielens datasets.}\label{tab:results}
\begin{center}
\begin{tabular}{rccccccccc}
\toprule
        & \multicolumn{3}{c}{$\mathbf{ml100k}$} & \multicolumn{3}{c}{$\mathbf{ml1m}$}   & \multicolumn{3}{c}{$\mathbf{ml10m}$}\\
        & \bf NMAE & \bf time(s) & \bf rank     & \bf NMAE & \bf time(s) & \bf rank     &\bf  NMAE & \bf time(s) & \bf rank \\ 
\midrule
{\bf ALT}     & $0.2165$ & $97$ & $93$     & $0.1806$ & $4133$ & $179$           & $0.1670$ & $205023
$ & $225
$\\
{\bf ALS-SI}  & ${0.1956}$ & $40$ & $16$     & ${0.1749}$ & $832$ & $31$           & ${0.1648}$ & $51205
$ & $36$\\
{\bf Ours}    & ${0.1959}$ & $2$  & $11$     & $0.1751$ & $39$  & $25$           & ${0.1659}$ & $3150$ & $41$\\
\hline
\end{tabular}
\end{center}
\vspace{-.3truecm}
\end{table}

%MASSI: I would remove this
%[Optimization considerations: We solve the factorization problem using conjugate gradient descent, adding 5 columns of noise at the end of each problem, if we haven't reached the global optimum yet. All methods are implemented in Matlab with all the operations that include $\Omega$ implemented in C. See Figure \ref{fig:convANDerrors}]

\section{Conclusions}
We studied the convergence properties of low rank factorization methods for trace norm regularization. Key to our study is a necessary and sufficient condition for global optimality, which can be applied to any critical points of the non-convex problem.  This condition together with a detailed analysis of the critical points lead us to propose a meta-algorithm for trace norm regularization, that incrementally expands the number of factors used by the non-convex solver. Although algorithms of this kind have been studied empirically for years, our analysis provides a fresh look and novel insights which can be used to confirm whether a global solution has been reached. Numerical experiments indicated that our optimality condition is useful in practice and the meta-algorithm is competitive with state-of-the art solvers. In the future it would be valuable to study improvements to our analysis, which would allow from one hand to derive precise rate of convergence for specific solvers used within the meta-algorithm and from another hand to study additional conditions under which our global optimality is guaranteed to activate immediately after the number of factors exceed the rank of the trace norm regularization minimizer. 
%Finally, it would be very interesting to study extensions of our analysis to other norms used for matrix factorization other than the trace norm. 

\section{Acknowledgments}

We kindly thank Andreas Argyriou, Kevin Lai and Saverio Salzo for their helpful comments, corrections and suggestions. This work was supported in part by EPSRC grant EP/P009069/1.

\iffalse
\newpage

Use of trace norm: \cite{Amit2007,Argyriou2008,Bach2008,Srebro2005}

General variational form for nuclear norms \cite{Jameson1987}

Factorization norms \cite{CMSM07}

Algorithms: \cite{dudik2012lifted,rennie2005fast}

Collaborative filtering approaches \cite{koren2009matrix}

Proximal methods: \cite{Bauschke2010}

Extensions \cite{bach2013convex,bouchard2013convex}; weighted trace norm \cite{foygel2011learning}.

Theoretical CS approaches: \cite{jain2013low,sun2016guaranteed,hardt2014understanding,bhojanapalli2016global}

Recent results showing all local minima are global: \cite{ge2016,Srebro2} and other attempts \cite{Vidal}

Andreas Maurer result for the PSD case \cite{maurer2008}
\fi

{\small
\bibliographystyle{unsrt}
\bibliography{mp10}

\begin{thebibliography}{10}

\bibitem{rennie2005fast}
J.~D.~M. Rennie and N.~Srebro.
\newblock Fast maximum margin matrix factorization for collaborative
  prediction.
\newblock In {\em Proceedings of the 22nd international conference on Machine
  learning}, pages 713--719, 2005.

\bibitem{koren2009matrix}
Y.~Koren, R.~Bell, and C.~Volinsky.
\newblock Matrix factorization techniques for recommender systems.
\newblock {\em Computer}, 42(8), 2009.

\bibitem{Argyriou2008}
A.~Argyriou, T.~Evgeniou, and M.~Pontil.
\newblock Convex multi-task feature learning.
\newblock {\em Machine Learning}, 73(3):243--272, 2008.

\bibitem{harchaoui2012large}
Z.~Harchaoui, M.~Douze, M.~Paulin, M.~Dudik, and J.~Malick.
\newblock Large-scale image classification with trace-norm regularization.
\newblock In {\em Proc. 2012 IEEE Conf. on Computer Vision and Pattern
  Recognition}, pages 3386--3393, 2012.

\bibitem{Amit2007}
Y.~Amit, M.~Fink, N.~Srebro, and S.~Ullman.
\newblock Uncovering shared structures in multiclass classification.
\newblock In {\em Proceedings of the 24th International Conference on Machine
  Learning}, 2007.

\bibitem{Bach2008}
F.~Bach.
\newblock Consistency of trace norm minimization.
\newblock {\em Journal of Machine Learning Research}, Vol. 8:1019--1048, 2008.

\bibitem{Srebro2005}
N.~Srebro, J.~D.~M. Rennie, and T.~S. Jaakkola.
\newblock Maximum-margin matrix factorization.
\newblock {\em Advances in Neural Information Processing Systems 17}, 2005.

\bibitem{Bauschke2010}
H.~H. Bauschke and P.~L. Combettes.
\newblock {\em Convex Analysis and Monotone Operator Theory in Hilbert Spaces}.
\newblock Canadian Mathematical Society, 2010.

\bibitem{hastie2015matrix}
T.~Hastie, R.~Mazumder, J.~D. Lee, and R.~Zadeh.
\newblock Matrix completion and low-rank svd via fast alternating least
  squares.
\newblock {\em J. Mach. Learn. Res}, 16(1):3367--3402, 2015.

\bibitem{NIPS2016_6048}
R.~Ge, J.~D. Lee, and T.~Ma.
\newblock Matrix completion has no spurious local minimum.
\newblock In {\em Advances in Neural Information Processing Systems 29}, pages
  2973--2981. 2016.

\bibitem{bhojanapalli2016global}
S.~Bhojanapalli, B.~Neyshabur, and N.~Srebro.
\newblock Global optimality of local search for low rank matrix recovery.
\newblock In {\em Advances in Neural Information Processing Systems}, pages
  3873--3881, 2016.

\bibitem{hsieh2014}
C.-J. Hsieh and P.~Olsen.
\newblock Nuclear norm minimization via active subspace selection.
\newblock In {\em International Conference on Machine Learning}, pages
  575--583, 2014.

\bibitem{Abernethy2009}
J.~Abernethy, F.~Bach, T.~Evgeniou, and J.-P. Vert.
\newblock A new approach to collaborative filtering.
\newblock {\em Journal of Machine Learning Research}, 10:803--826, 2009.

\bibitem{Beck2009}
A.~Beck and M.~Teboulle.
\newblock A fast iterative shrinkage-thresholding algorithm for linear inverse
  problems.
\newblock {\em SIAM J. Imaging Sciences}, 2(1):183--202, 2009.

\bibitem{dudik2012lifted}
M.~Dudik, Z.~Harchaoui, and J.~Malick.
\newblock Lifted coordinate descent for learning with trace-norm
  regularization.
\newblock In {\em Artificial Intelligence and Statistics}, pages 327--336,
  2012.

\bibitem{Jameson1987}
G.~J.~O. Jameson.
\newblock {\em Summing and Nuclear Norms in Banach Space Theory}.
\newblock Cambridge University Press, 1987.

\bibitem{Bertsekas1999}
D.~P. Bertsekas.
\newblock {\em Nonlinear programming}.
\newblock Athena scientific Belmont, 1999.

\bibitem{lee2016}
J.~D. Lee, M.~Simchowitz, M.~I. Jordan, and B.~Recht.
\newblock Gradient descent only converges to minimizers.
\newblock In {\em 29th Annual Conference on Learning Theory}, pages 1246--1257,
  2016.

\bibitem{Vidal}
B.~D. Haeffele and R.~Vidal.
\newblock Global optimality in tensor factorization, deep learning, and beyond.
\newblock {\em arXiv preprint arXiv:1506.07540}, 2015.

\bibitem{woodruff2014}
D.~P Woodruff et~al.
\newblock Sketching as a tool for numerical linear algebra.
\newblock {\em Foundations and Trends{\textregistered} in Theoretical Computer
  Science}, 10(1--2):1--157, 2014.

\bibitem{murty1987}
K.~G. Murty and S.~N. Kabadi.
\newblock Some np-complete problems in quadratic and nonlinear programming.
\newblock {\em Mathematical programming}, 39(2):117--129, 1987.

\bibitem{boyd2004}
S.~Boyd and L.~Vandenberghe.
\newblock {\em Convex Optimization}.
\newblock Cambridge University Press, 2004.

\bibitem{attouch2010}
H.~Attouch, J.~Bolte, P.~Redont, and A.~Soubeyran.
\newblock Proximal alternating minimization and projection methods for
  nonconvex problems: An approach based on the kurdyka-{\l}ojasiewicz
  inequality.
\newblock {\em Mathematics of Operations Research}, 35(2):438--457, 2010.

\bibitem{bolte2010}
J.~Bolte, A.~Daniilidis, O.~Ley, and L.~Mazet.
\newblock Characterizations of {\l}ojasiewicz inequalities: subgradient flows,
  talweg, convexity.
\newblock {\em Transactions of the American Mathematical Society},
  362(6):3319--3363, 2010.

\bibitem{attouch2009}
H.~Attouch and J.~Bolte.
\newblock On the convergence of the proximal algorithm for nonsmooth functions
  involving analytic features.
\newblock {\em Mathematical Programming}, 116(1):5--16, 2009.

\bibitem{harper2016}
F.~M. Harper and J.~A. Konstan.
\newblock The movielens datasets: History and context.
\newblock {\em ACM Transactions on Interactive Intelligent Systems (TiiS)},
  5(4):19, 2016.

\bibitem{Lewis1995}
A.~S. Lewis.
\newblock The convex analysis of unitarily invariant matrix functions.
\newblock {\em Journal of Convex Analysis}, 2:173--183, 1995.

\end{thebibliography}
}

\newpage
\appendix

\section*{Appendix}

\renewcommand{\thetheorem}{\Alph{section}.\arabic{theorem}}

Here we collect some auxiliary results and we provide proofs of the results stated in the main body of the paper. 
%MASSI set the notation
\section{Auxiliary Results}
\label{sec:app-aux}
The first lemma establishes the variational form for the trace norm; its proof can be found in \cite{Jameson1987}.
\begin{lemma}[Variational Form of the Trace Norm]\label{lemma:trace_norm}
For every $W\in \R^{n \times m}$ and $r \in \N$ let ${\cal F}_r(W)= \{(A,B) \in\R^{n \times k}\times\R^{m \times r}~:~ AB\trans = W\}$. Let $k= \rank(W)$ and let $\sigma_1(W) \geq \cdots \geq \sigma_k(W) >0$ be the $k$ singular values of $W$. Then
%    \|W\|_* &= &  \inf \left\{\|A\|_F \|B\|_F~: ~ r\in\N,~ A\in\R^{n \times r},~ B\in\R^{m \times r},~ AB\trans = W\right\} \\
\iffalse
\begin{align}
\label{eq:id1}
         \|W\|_* = \sum_{i=1}^k \sigma_i(W) & =  \inf \left\{\sum_{i=1}^r \|a_i\|\|b_i\|~~  \Big| ~~([a_1\cdots a_r],[b_1\cdots b_r]) \in {\cal F}_r(W),~r \in \N \right\}~~\\
         \label{eq:id2}
        & =  \inf \left\{~ \|A\|_F \|B\|_F~ \Big| ~ (A,B) \in {\cal F}_r(W),~r \in \N \right\} \\
        \label{eq:id3}
    & = \frac{1}{2} \inf \left\{~\|A\|_F^2 + \|B\|_F^2~ \Big| ~ (A,B) \in {\cal F}_r(W),~r \in \N \right\}.
\end{align}
\fi
\[
\|W\|_* = \sum_{i=1}^k \sigma_i(W)=\frac{1}{2} \inf \left\{~\|A\|_F^2 + \|B\|_F^2~ \Big| ~ (A,B) \in {\cal F}_r(W),~r \in \N \right\}.
\]
Furthermore if $W=U\Sigma V\trans$ is
a singular value decomposition (SVD) for $W$, with $\Sigma=\diag(\sigma_1(W),\dots,\sigma_r(W))$, the infimum is attained for $r = \rank(W)$, $A= U\Sigma^\frac{1}{2}$, and $B=V\Sigma^\frac{1}{2}$.  
\end{lemma}
\iffalse
\begin{proof}
We first note that if $(U,\Sigma,V)$ is an SVD of $W$ then $\|W\|_* = \|\Sigma\|_*$. Hence it is enough to prove the result in the diagonal case. We have 
$$
\sum_{i=1}^k \sigma_i(W) = \trace(\Sigma) = \trace (AB\trans) = \sum_{i=1}^r \trace (a_i b_i\trans) = \sum_{i=1}^r a_i\trans b_i \leq \sum_{i=1}^r \|a_i\|\|b_i\|
$$
where the last step follows by Cauchy-Schwarz's inequality. Moreover, let $e_1,\dots,e_n\in\R^n$ 
%and $f_1,\dots,f_m\in\R^m$ 
denote 
%respectively 
the canonical basis of $\R^n$. Then the above inequality is tight for $a_i = b_i = e_i \sqrt{\sigma_i}$,  
%and $b_i = e_i \sqrt{\sigma_i}$, 
therefore proving the identity at \eqref{eq:id1}.

Two prove the second identity, we use again Cauchy-Schwarz's inequality to obtain, for every $(A,B) \in \R^{n \times r} \times \R^{m \times r}$ that
\[
\sum_{i=1}^r \|a_i\|_2\|b_i\|_2  \leq \|A\|_{F}\, \|B\|_{F} 
\]
the result then follows from \eqref{eq:id1}.

Finally, \eqref{eq:id3} follows by the arithmetic-geometric mean inequality
\end{proof}
\fi

Recall that if $\phi:\R^d \rightarrow \R \cup \{+\infty\}$ is proper convex function, its sub-differential at $x$ is the set 
\[
\partial \phi(x) =\left\{u : \phi(x) + \lb u,y-x\rb \leq \phi(y),~\textrm{for~all}~y \in \textrm{domain}(\phi)\right\}.
\]
The elements of $\partial \phi(x)$ are called the sub-gradients of $\phi$ at $x$.

Let $\Ob_n$ be the set of $n \times n$ orthogonal matrices. A norm $\|\cdot\|: \R^{m \times n} \rightarrow [0,\infty)$ is called {\em orthogonally invariant}
if, for every $U \in \Ob_n$, $V \in \Ob_m$ and $W \in \R^{n \times m}$ we have that $\|UWV\|=\|W\|$ or, equivalently $\|W\|= g(\sigma(W))$, where $g$ is a symmetric gauge function (SGF), that is $g$ is a norm invariant to permutations and sign changes. An important example of orthogonally invariant norms are the $p$-Schatten norms, $\|W\|= \|\sigma(W)\|_p$, where $\|\cdot\|$ is the $\ell_p$-norm of a vector. In particular, for $p\in \{1,2,\infty\}$ we have the trace, Frobenius, and spectral norms, respectively.

The following result is due to \cite[Cor.~2.5]{Lewis1995}.
\begin{lemma}
If $\|\cdot\|: \R^{m \times n}$ is an orthogonally invariant and $g$ is the associated SGF, then for every $W \in \R^{n \times m}$, it holds that
\[
\partial \|W\| = \{U\diag(\mu) V\trans~:~U \in \Ob_n,~V\in \Ob_m,~ \mu \in \partial g(\sigma),~W=U\diag(\sigma) V\trans\}.
\]
\label{lem:lewis}
\end{lemma}

\section{Proofs}
\label{sec:app-derivations}
For convenience of the reader, we restate the results presented in the main body of the paper.

\PEquivalence*
\begin{proof}

Let $W_*\in\R^{n \times m}$ be a minimizer for $\fla$ of rank $r_* = \rank(W_*)$ and let $U \Sigma V^\top$ be a singular value decomposition of $W_*$ with $U\in\R^{n \times r_*}$ and $V\in\R^{m \times r_*}$ with orthonormal columns and $\Sigma\in\R^{r_* \times r_*}$ diagonal with positive diagonal entires. Define $A_* = U \Sigma^{1/2} \in\R^{n \times r_*}$ and $B = V \Sigma^{1/2}\in\R^{m \times r_*}$. By construction $\|W_*\|_* = \frac{1}{2}(\|A_*\|_F^2 + \|B_*\|_F^2)$ and therefore 
\[
  \fla(W_*) = \fla(A_*B_*^\top) = \glarstar(A_*,B_*).
\]
Now, we prove that $(A_*,B*)$ is a minimizer for $\glarstar$. Suppose by contraddiction that there exist a couple $A_1\in\R^{n \times r_*},B_1\in\R^{m \times r_*}$ such that $\glarstar(A_1,B_1) < \glarstar(A_*,B_*)$. Define
\[
(\bar A_1, \bar B_1) = \argmin{AB^\top = A_1B_1^\top}{ \|A\|_F^2 + \|B\|_F^2}.
\]
Then by \autoref{lemma:trace_norm} we have 
\[
  \|\bar A_1 \bar B_1^\top\|_* = \frac{1}{2}(\|\bar A_1\|_F^2 + \|\bar B_1\|_F^2) \leq \frac{1}{2}(\|A_1\|_F^2 + \|B_1\|_F^2)
\]
and therefore 
\[
  \fla(\bar A_1\bar B_1^\top) = \glarstar(\bar A_1,\bar B_1) \leq \glarstar(A_1,B_1) < \glarstar(A_*,B_*) = \fla(W_*)
\]
which is clearly not possible since $W_*$ was a global minimizer for $W_*$.
\end{proof}

\PCharacterization*
\begin{proof}
%MMM: i would consider to change the notation tilde A,B  to hat A,B and later current hat A,B to e.g. M, P
Let $\hA \in \mathbb{R}^{m \times r}$, $\hB \in \mathbb{R}^{n \times r}$ be the matrices that correspond to a critical point of $\glar$. We let 
\beq
\nabla  \ell(\hA \hB\trans) = \lambda U_1 V_1\trans +  U_2 \Sigma_2 V_2\trans + 
U_3 \Sigma_3 V_3\trans \label{SVDbreakdown}
\eeq
be the breakdown SVD of the gradient of $\ell$ at $\hA \hB\trans$, where $\Sigma_2$ is the diagonal matrix formed by the singular values strictly larger than $\lambda$ and $\Sigma_3$ is the diagonal matrix formed by the singular values strictly smaller than $\lambda$ (including those which are zero). For each $i=1,2,3$ we denote with $s_i$ the number of columns of $U_i\in\R^{n \times s_i}$ and $V_i\in\R^{m \times s_i}$. Recall that the matrices $[U_1~U_2~U_3]\in\R^{n \times n}$ and $[V_1~V_2~V_3]\in\R^{m \times m}$ are both orthogonal.

Taking the derivatives of \eqref{eq:factorized-problem} w.r.t. $A$ and $B$ and setting them to zero gives the following optimality conditions for the critical points
\begin{align}
  \nabla  \ell(\hA \hB\trans) \hB + \lambda \hA &= 0 \label{wrtA}\\ 
  \nabla  \ell(\hA \hB\trans)\trans \hA+ \lambda \hB &= 0 \label{wtrB}.
\end{align}
Solving \eqref{wtrB} for $\hB$ and 
and replacing it in \eqref{wrtA} yields that
\beq
\left( - \frac{1}{\lambda^2} \nabla  \ell(\hA \hB\trans) \nabla  \ell(\hA \hB\trans)\trans + I_m\right) \hA = 0.
\label{eq:mmm}
\eeq
By \eqref{SVDbreakdown} $  \ell(\hA \hB\trans) \nabla  \ell(\hA \hB\trans)\trans = \lambda^2 U_1\trans + U_2 \Sigma_2^2 U_2\trans + U_3 \Sigma_3^2 U_3\trans$. Using this in \eqref{eq:mmm} and rearranging gives
\[
\left( I_m -U_1 U_1\trans - U_2 \frac{\Sigma_2^2}{\lambda^2} U_2\trans - U_3 \frac{\Sigma_3^2}{\lambda^2} U_3\trans\right) \hA= 0
\]
which we rewrite as
\[
\left( U_2 (I-\frac{\Sigma_2^2}{\lambda^2})U_2\trans + U_3 (I-\frac{\Sigma_3^2}{\lambda^2}) U_3\trans\right) \hA= 0.
\]
Therefore, the columns of $\hA$ must be in the range of $U_1$ (i.e. orthogonal to $U_2$ and $U_3)$, namely
\beq
\hA= U_1 C
\label{eq:spanA}
\eeq
for some $C \in \R^{s_1 \times r}$.  Similarly we derive that
\beq
\hB= V_1 D
\label{eq:spanB}
\eeq
for some $D \in \R^{s_1 \times r}$. Combining \eqref{SVDbreakdown} with \eqref{eq:spanB} we obtain that 
\[
\nabla  \ell(\hA \hB\trans) \hB = -\lambda U_1 C.
\]
Using this equation, \eqref{eq:spanA} and \eqref{eq:spanB} we rewrite \eqref{wrtA} as $\lambda U_1 D + \lambda U_1 C = 0$. This implies that $D=-C$ and, so, $\hB= - V_1 C$.
\end{proof}

\TCriterion*
\begin{proof}
Let $\hW = \hA \hB\trans$. We need to show that $0 \in \partial f_\lambda(\hW)$
or, equivalently
\beq
- \frac{1}{\lambda} \nabla  \ell(\hW) \in \partial \|\hW\|_*.
\label{opt:con}
\eeq
Let $Z = - \frac{1}{\lambda} \nabla  \ell(\hW)$. As a special case of Lemma \ref{lem:lewis} we have that $Z \in \partial \|\hW\|_*$ holds true if and only if there exists a simultaneous singular value decomposition of the form $\hW = U \diag(\sigma) V\trans$, $Z = U \diag(\sigma(Z)) V\trans$ and $\sigma(Z) \in \partial \|\sigma(\hW)\|_1$, with $\sigma(Z)$ denoting the spectrum of $Z$ (namely the vector of singular values of $Z$ arranged in non-increasing order). Recall that 
$$
\partial \|\sigma\|_1 = \{z \in \R^m : z_i = 1~\text{if}~\sigma_i \neq 0,~\text{and}~z_i \in [-1,1]~\text{otherwise}\}.
$$

%Thus we need to show that $Z$ and $\hW$ have a simultaneous singular value decomposition and all sigular values of $Z$ corresponding to the nonzero singular values of $\hW$ are equal to $1$, while the remaining ones are not greater than $1$.
Using the same notation of \autoref{thm:characterization}, consider the SVD 
\[
Z = - U_1 V_1\trans - U_2 \frac{\Sigma_2}{\lambda} V_2\trans - U_3 \frac{\Sigma_3}{\lambda} V_3\trans.
\]
By \autoref{thm:characterization} we have that $A_* = U_1 C$ and $B_* = - V_1 C$, with $C \in\R^{s_1\times r}$. Now, let $\bar r = \textrm{rank}(C)\leq r$ and let $C = P \Gamma Q^\top$ be the SVD of $C$, with $P\in\R^{s_1 \times \bar r}$ and $Q\in\R^{\bar r \times r}$ matrices with orthonormal columns and $\Gamma\in\R^{\bar r \times \bar r}$ diagonal with positive diagonal elements. Denote $\widetilde P = [P ~P^{\SSS \perp}]\in\R^{s_1 \times s_1}$ the orthonormal matrix obtained by completing $P$ with a matrix $P^{\SSS \perp}\in\R^{s_1 \times (s_1 - \bar r)}$ with orthonormal columns such that $P\trans P^{\SSS \perp} = 0$. Moreover denote $\widetilde \Gamma \in\R^{s_1 \times s_1}$ as 
\[
  \widetilde \Gamma = \left[\begin{array}{cc}\Gamma & 0_{\bar r \times s_1 - \bar r} \\ 0_{s_1 - \bar r \times \bar r} & 0_{s_1 - \bar r \times \bar s_1 - \bar r}\end{array}\right].
\]
Then,  
\[
\hW = - U_1 CC\trans V_1 = (-U_1P)\Gamma^2 (V_1P)\trans = (-U_1\widetilde P) \widetilde\Gamma^2(V_1 \widetilde P)\trans = (-\widetilde U_1)\widetilde\Gamma^2(\widetilde V_1)
\]
with $\widetilde U_1 = U_1\widetilde P$ and $\widetilde V = V_1\widetilde P$. Note that since $\widetilde P$ is orthonormal, $\widetilde U_1 \widetilde V_1\trans = U_1V_1\trans$. Moreover $U_2,U_3$ have columns orthogonal to $\widetilde U_1$ and $V_2,V_3$ have columns orthogonal to $\widetilde V_1$. Consequently,
\[
Z = - U_1 V_1\trans - U_2 \frac{\Sigma_2}{\lambda} V_2\trans - U_3 \frac{\Sigma_3}{\lambda} V_3\trans = - {\tilde U}_1 {\tilde V}_1\trans - U_2 \frac{\Sigma_2}{\lambda} V_2  - U_3 \frac{\Sigma_3}{\lambda} V_3\trans
\]
is an alternative singular value decomposition for $Z$. Therefore, $\hW$ and $Z$ have a simultaneous singular value decomposition and we can conclude that $\sigma(Z) \in \partial \|\sigma(\hW)\|_1$ if and only if $s_2 = 0$, namely $\|\grad\ell(W_*)\|\leq\lambda$ as desired.
\end{proof}

\PEscape*
\begin{proof}
By Theorem \ref{thm:characterization}, $\hA=U_1 C$ and $\hB= - V_1 C$, where $U_1$ and $V_1$ are the matrices of left and right singular vectors of $\nabla \ell (\hA \hB\trans)$ and $C \in \R^{s \times r}$, for $s=\rank(\hA) = \rank(\hB)$.~Taking the SVD of $C = P \Gamma Q\trans$, we rewrite
\begin{align}
\hA = U_1 P \Gamma Q\trans, \quad \hB = - V_1 P \Gamma Q\trans. \label{trans}
\end{align}
Since $\hA$ and $\hB$ are rank deficient and they have the same null space, we can choose $q \in \R^{r}$ such that $\hA q = \hB q = 0$. Let $u_2$ and $v_2$ the a left and right singular vector of $\nabla  \ell (\hA \hB\trans)$ with singular value equal to $\mu$. 

We consider a perturbation of the objective function in the direction $(-u_2q\trans,v_1q\trans)$. We have that
\begin{align}
L(\gamma) &=  \ell\big((\hA - \gamma u_2q\trans) (\hB + \gamma v_1q\trans)\trans\big) + \frac{\lambda}{2}\big (\|\hA - \gamma u_1q\trans\|_F^2 + \|\hB + \gamma v_1q\trans\|_F^2\big) \\
~ &=  \ell(\hA \hB\trans - \gamma^2 u_2 v_2\trans) + \frac{\lambda}{2} (\|\hA\|_F^2 + \|\hB\|_F^2) + \lambda \gamma^2. \\
\end{align}
Thus, we have that
\begin{align}
L'(\gamma) &= \langle -2 \gamma \nabla  \ell(\hA \hB\trans - \gamma^2 u_2 v_2\trans), u_2 v_2\trans \rangle + 2 \gamma \lambda \\
& = - 2 \gamma u_2\trans \nabla  \ell(\hA \hB\trans - \gamma^2 u_2 v_2\trans) v_2+ 2 \gamma \lambda.
\end{align}
Consequently
\[
L''(0) = 2 (\lambda - u_2\trans\nabla  \ell(\hA \hB\trans) v_2) = 2 (\lambda-\mu) < 0
\]
and the result follows.
\end{proof}

\TConvergence*
\begin{proof}
We have shown in \autoref{thm:escape} that every critical point $(A,B)$ of $\glar$ for $r\geq\min(n,m)$ is either a global minimizer or a strict saddle point, namely such that the Hessian $\grad^2\glar(A,B)$ has at least a negative eigenvalue. Therefore, since the error function $\ell$ is twice differentiable with gradient Lipschitz continuous also $\glar$ is. Let $L_r>0$ be the Lipschitz constant of $\grad\glar$. Then, we are in the hypotheses of Thm.$4.1$ in \cite{lee2016} which states that if $(A_k,B_k)_{k\in\N}$ is obtained with step $0<\alpha<1/L_r$ with initial point $(A_0,B_0)$ sampled uniformly at random, then for any $(A_*,B_*)$ {\em strict saddle point} of $\glar$,
\[
  \textrm{Prob}\left(\lim_{k\to+\infty} (A_k,B_k) = (A_*,B_*)~\right) = 0
\]
which implies the desired result and corresponds to \autoref{cor:6}. 
\end{proof}
At last we show that if the spectral norm of the gradient at a critical point is close to one from above then value of the objective function is close to the global minimum.
\begin{proposition}
Let $(A_*,B_*)$ be a critical point of $\glar$ and let $W_* = A_* B_*\trans$. If $\|\grad \ell(W_*)\| \leq \lambda+\epsilon$  with $\epsilon \in [0,\lambda]$ then
\[
f_\lambda(W_*)  \leq \min_W f_\lambda(W) + \epsilon  \left(\|W_*\|_{*} + \frac{\ell(0)}{
\lambda}\right)
\]
\end{proposition}
\begin{proof}
We write  
\begin{eqnarray}
\nabla \ell(W_*) & =  & \lambda U_1 V_1\trans+  U_2 \Sigma_2 V_2\trans + 
\lambda U_3 V_3\trans + U_3 \Sigma_3 V_3\trans \\ \nonumber
& = & \lambda (U_1 V_1\trans +U_3 V_3\trans) +  U_2 \Sigma_2 V_2\trans + 
\lambda U_3 V_3\trans + U_3 (\Sigma_3 - \lambda U) V_3\trans. 
\end{eqnarray}
Let $Z =U_3 (\Sigma_3 - \lambda U) V_3\trans$. By assumption $\|\Sigma_3\|\leq 2\lambda$, implying that $-Z \in \partial f_\lambda(W_*)$. That is, for every $W \in \Mnm$,
\beq
\label{eq:ww}
f_\lambda(W_*) + \lb Z, W-W_*\rb \leq f_\lambda(W).
\eeq
In turn this implies that 
\[
f_\lambda(W_*) \leq \lb Z, W-W_* \rb + f_\lambda(W).
\]
Since the minimizer can be constrainted to be in the set $\{W: \|W\|_{\rm tr} \leq \ell(0)/
\lambda\}$ we conclude from \eqref{eq:ww} that
\begin{eqnarray}
\nonumber
f_\lambda(W_*) & \leq & f_\lambda(W) -\lb Z, W-W_* \rb \\ [2pt]
\nonumber
& \leq &\min_W f_\lambda(W) + \|Z\| \left(\|W_*\|_* + \frac{\ell(0)}{
\lambda}\right)\\
\nonumber
& \leq & 
\min_W f_\lambda(W) + \epsilon  \left(\|W_*\|_{\rm tr} + \frac{\ell(0)}{
\lambda}\right).
\end{eqnarray}
% We can also obtain
% $$
% f_\lambda(W_*)  \leq \min_W f_\lambda(W) + \epsilon \left( 2 \|A_*\|_F^2 + \ell(W_*)/\lambda \right)
% $$
\end{proof}

\section{Meta-algorithm with Explicit Escape from Stationary Points}

We expand on the discussion in \secref{sec:meta-algorithm}, where we formally introduced the meta-algorithm considered in this work. Specifically we observe that in the formulation of Algorithm \ref{alg:meta-algorithm} we did not exploit the result in \autoref{thm:escape}, providing an explicit decreasing direction for $\glarplus$ from the inflated point $([A_*~0],[B_* ~0])$, where $(A_*,B_*)$ is a stationary point for $\glar$. For completeness in Algorithm \ref{alg:meta-algorithm-descend} we report a variant of the meta-algorithm proposed in this paper that makes use of such escape direction. We care to point out that in our experiments we did not observe any statistically significant difference with the original version. 

We discuss here the details of Algorithm \ref{alg:meta-algorithm-descend}. By \autoref{thm:escape} we know that there exists $\gamma>0$ such that 
\[
  \glarplus([A_*~0]+\gamma[0_{n \times r},~u],[B_*~0] +\gamma[0_{m \times r}, ~-v]) < \glar(A_*,B_*)
\] 
where $u\in\R^n, v\in\R^m$ are respectively the left and right singular vectors associated to the largest singular value $\mu$ of $\grad \ell(A_* B_*\trans)$. In Algorithm \ref{alg:meta-algorithm-descend} these are provided by the routine {\sc LargestSingularPair}.

We can find a new point on the direction specified above by trying to approximate the minimizer of 
\[
  \gamma_* = \argmin{\gamma \in [0,1]} ~~ \glarplus([A_*~0]+\gamma[0_{n \times r} ~u],[B_*~0] +\gamma[0_{m \times r} ~-v])
\]
for instance by considering a set of candidate $\gamma_1,\dots,\gamma_M\in(0,1]$ (e.g. $\gamma_i = \gamma_0^i$ for $i=1,\dots,M$) and choose the one leading to the lowest value for $\glarplus$. In Algorithm \ref{alg:meta-algorithm-descend} we refer to this procedure as {\sc LineSearch} given its analogy with standard line search approaches often used in optimization. However notice that such methods can typically leverage on a criterion depending on the norm of the gradient $\|\grad\glarplus([A_*~0],[B_*~0])\|_F$ in order to guarantee that the function decreases significantly. We cannot replicate such strategy since $\grad\glarplus([A_*~0],[B_*~0]) = 0$ in our case.

As a final observation, we care to point out (as we already done in \secref{sec:meta-algorithm}) that it is necessary to perturb the point $([A_*~\gamma u], [B_*~-\gamma v])$, e.g. by adding some noise, in order to guarantee that $([A_*~\gamma u], [B_*~-\gamma v])$ does not belong to the set of measure zero, mentioned in \autoref{thm:global-convergence}, of initial points that converge to strict saddle points of $\glar$.

\begin{algorithm}[t]
   \caption{\textsc{Meta-Algorithm with Explicit Escape From Stationary Points}}
   \label{alg:meta-algorithm-descend}
\begin{algorithmic}

   \State ~

   \State {\bfseries Input:} $\lambda>0$, $\epsilon_{\textrm{conv}}>0$ convergence tolerance, $\epsilon_{\textrm{crit}}>0$ global criterion tolerance, $\sigma>0$ noise parameter.

   \State {\bfseries Initialize:} Set $r=1$. Sample $A_0'\in\R^n$ and $B_0'\in\R^n$ randomly.

%\vspace{.071truecm}

   \State {\bfseries For} $r = 1$ to $\min(n,m)$

   \State \quad $(A_r,B_r)$ = {\sc OptimizationAlgorithm($A_{r-1}',B_{r-1}',\glar,\epsilon_{\rm{conv}}$) }

   \State \quad {\bfseries If} $\|\grad \ell(A_r B_r\trans)\|\leq\lambda + \epsilon_{\textrm{crit}}$
   \State \quad \quad {\bfseries Break} 
   \State \quad $[u,\mu,v] = ${\sc LargestSingularPair}$(\grad \ell(A_r B_r\trans))$
   \State \quad $\hat \gamma = $ {\sc LineSearch($[A_r~ 0], [B_r~ 0], u, v$)}
   \State \quad Perturb $u = u + \eta$ with $\eta\sim\mathcal{N}(0,\sigma I_{n\times n})$
   \State \quad Perturb $v = v + \eta$ with $\eta\sim\mathcal{N}(0,\sigma I_{m\times m})$
   \State \quad $(A_{r+1}',B_{r+1}') = ([A_r~ \gamma u], [B_r~ - \gamma v])$
   \State \quad $r = r + 1$
   \State {\bfseries End}

%\vspace{.071truecm}
   
   \State {\bfseries Return $(A_r,B_r)$} 

\end{algorithmic}
\end{algorithm}

\section{On the  Kurdyka-Lojasiewicz Inequality}

We extend here the discussion on the KL inequality (\autoref{def:kl}) and corresponding convergence results reviewed in \secref{sec:rates}. As a special case to the problem of optimizing $\glar$ considered in this work, we have recalled in \autoref{thm:rates} that if $\glar$ satisfies the KL inequality, we can expect GD to exhibit polynomial rates of convergence. A natural question is when $\glar$ satisfies such inequality. To provide an insight on this issue, we consider the result in \cite{bolte2010} showing that the KL inequality is satisfied by {\em semi-algebraic} functions. We recall that a set $S\subseteq\R^d$ is said semi-algebraic if there exists a finite number of polynomials $p_{kh},q_{kh}:\R^d\to\R$ such that 
\[
  S = \bigcup_{k=1}^K\bigcap_{h=1}^H \left\{ x\in\R^d ~ \Large| ~ p_{kh}(x) = 0,~ q_{kh}(x)\leq0\right\}.
\]
A function $f:\R^d\to\R$ is said semi-algebraic if its graph
\[
  \textrm{graph}~f = \left\{(x,t) ~ \Large| ~ x\in\R^d, t\in\R, f(x) = t\right\}
\]
is semi-algebraic.

Note that a variety of error functions $\ell:\R^{n \times m}\to\R$ typically used in machine learning and matrix factorization problems are semi-algebraic (e.g. the square loss). Interestingly, we have that if $\ell$ is semi-algebraic, then $\ell_r:\R^{n \times r}\times\R^{m \times r}\to\R$ such that $\ell_r(A,B) = \ell(AB^\top)$, is semi-algebraic as well. Indeed, by definition of semi-algebraic function we have that 
\[
  \textrm{graph}~\ell = \left\{(X,t) ~ \Large| ~ X\in\R^{n \times m}, t\in\R, \ell(X) = t\right\}
\]
is semi algebraic, therefore there exist $p_{kh},q_{kh}:\R^{n \times m}\times\R\to\R$ such that
\[
  \textrm{graph}~\ell = \bigcup_{k=1}^K\bigcap_{h=1}^H \left\{ (X,t) ~ \Large| ~ p_{kh}(X,t) = 0, q_{kh}(X,t)\leq0\right\}.
\]
Now, denote $X_{ij}$ the $(i,j)$-th entry of $X$. For $X = AB^\top$ with $A\in\R^{n\times r}$ and $B\in\R^{m \times r}$ we have $X_{ij} = \sum_{s=1}^r A_{is}B_{js}$ namely $X_{ij} = m_{ij}(A,B)$ with $m_{ij}:\R^{n \times r}\times\R^{m \times r}\to\R$ a real polynomial. Let us denote $m:\R^{n\times r}\times\R^{m\times r}\to\R^{n \times m}$ the matrix valued function with $(i,j)$-th entry $m(A,B)_{ij}=m_{ij}(A,B)$. Since every $p_{kh}$ and $q_{kh}$ are polynomial in the variables of $X$, we have that $p_{kh} \circ m$ and $h_{kh} \circ m$ are still polynomials and therefore the graph of $\ell_r$
\begin{align*}
  \textrm{graph}~\ell_r & = \Large\{(A,B,t) ~ \Large| ~ A\in\R^{n \times r}, B\in\R^{m \times r}, t\in\R, \ell(AB^
  top) = t\Large\} \\
  & = \bigcup_{k=1}^K\bigcap_{h=1}^H \Large\{ (A,B,t) ~ \Large| ~ p_{kh}(m(A,B)),t) = 0, ~q_{kh}(m(A,B),t)\leq0\Large\}\\
\end{align*}
is semi-algebraic, which implies that $\ell_r$ is a semi-algebraic function.

Going back to the question of whether $\glar(A,B) = \ell(AB^\top) + \frac{\lambda}{2}(\|A\|_F^2 + \|B\|_F^2)$ is semi-algebraic, we now can conclude that it is sufficient to assume the error function $\ell$ to be semi-algebric. Indeed, it is well-known (see e.g. \cite{bolte2010}) that the squared Frobenius norm of a matrix is semi-algebraic and that the finite sum of semi-algebraic functions is still semi-algebraic. Therefore we can re-formulate \autoref{thm:rates} in terms only of the error $\ell$, namely 
\begin{corollary}[Convergence Rate of Gradient Descent]
Let $(A_k,B_k)_{k\in\N}$ a sequence produced by GD method applied to $\glar$. If $\ell$ is semi-algebraic, then $\glar$ satisfies the KL inequality for some constant $\alpha \in[0,1)$, and there exists a critical point $(A_*,B_*)$ of $\glar$ and constants $C>0$, $b\in(0,1)$ such that
\begin{equation}
\label{eq:gd-rate}
\|(A_k,B_k) - (A_*,B_*)\|_F^2 \leq \left\{
\begin{array}{ll}
Cb^k & \text{if }  \alpha \in(0,1/2],\\
C k^{-\frac{1-\alpha}{2\alpha - 1}} & \text{if } \alpha\in(1/2,1).
\end{array} \right.
\end{equation}
Furthermore, if $\alpha = 0$ convergence is achieved in a finite number of steps.
\end{corollary}

\section{Further Experiments}
This last section provides more comparative experiments between the meta-algorithm and the two state-of-the art solvers, as well as a comparison to the algorithm in \cite{Vidal} 

\begin{figure}[t]
\centering
\includegraphics[width=0.95\textwidth]{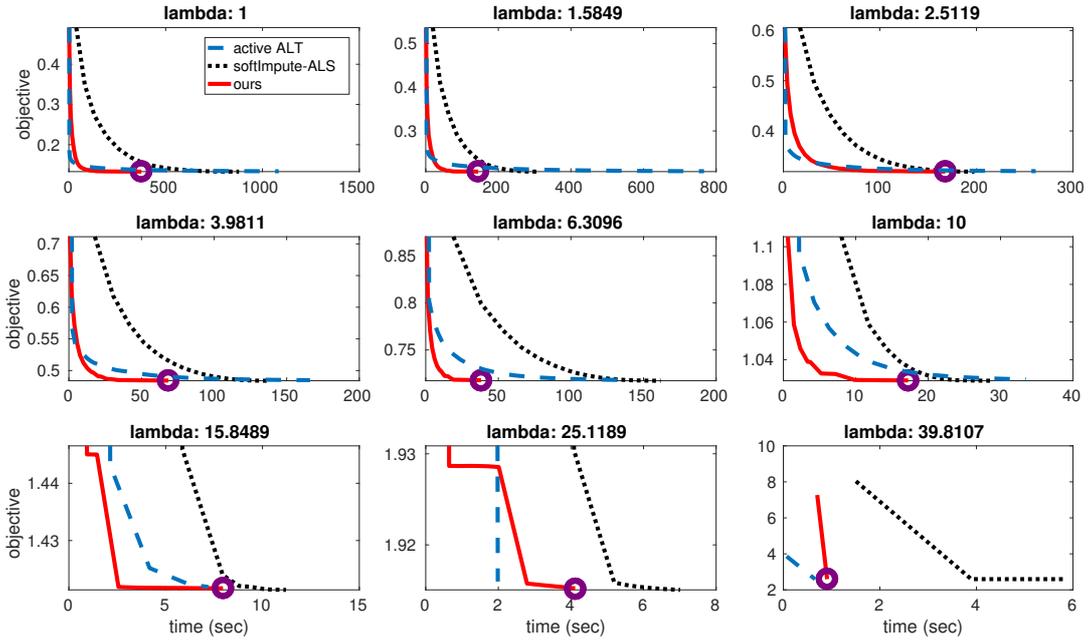}
\caption{Convergence of the objective function on $ml100k$ for various lambda values. The center plot is for the optimal lambda value based on the validation errors. The circle indicates that our proposed global optimality criterion has been satisfied. \label{fig:ml100k-optimization}}
\end{figure}
\begin{figure}[!h]
\centering
\includegraphics[width=0.95\textwidth]{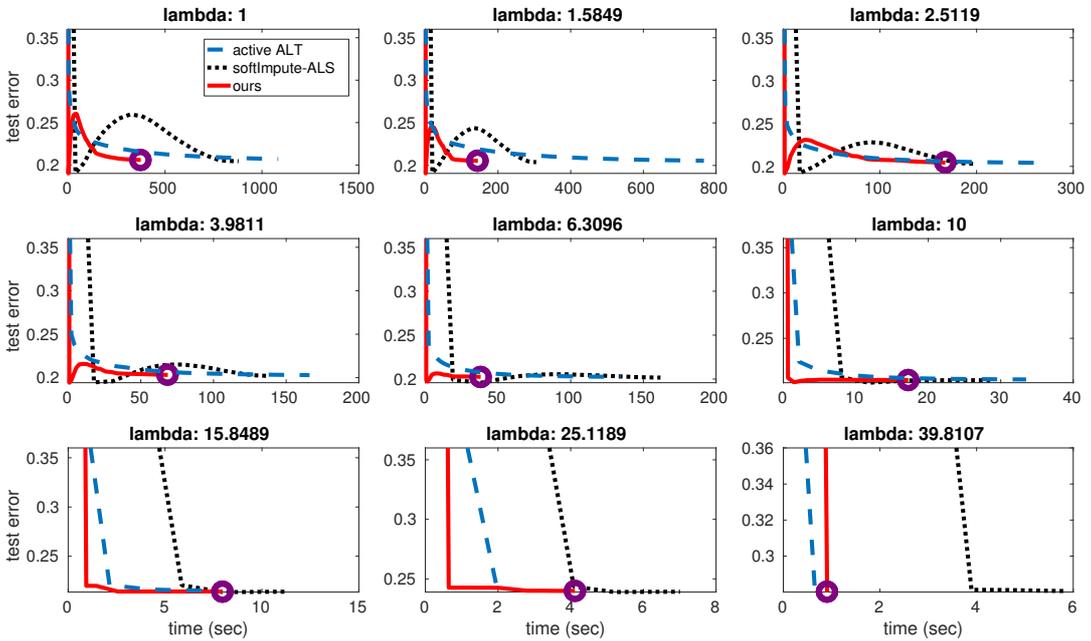}
\caption{Test error on $ml100k$ for various lambda values. The center plot is for the optimal lambda value based on the validation errors. The circle indicates that our proposed global optimality criterion has been satisfied on the original problem. \label{fig:ml100k-test}}
\end{figure}

\subsection{Large Scale Matrix Completion}

In \secref{sec:exps} we reported on the performance of the three methods considered for $\lambda$ chosen by validation as the parameter leading to the lowest Normalized Mean Average Error (NMAE). For completeness here we report the same experiments for a range of candidate values of $\lambda$.

Figures~\ref{fig:ml100k-optimization} and \ref{fig:ml100k-test} report respectively the value of the objective function $\fla$ and the test error (NMAE) for the three methods considered in our experiments. Interestingly we observe a similar pattern to the one of the optimal $\lambda$, with our method exhibiting comparable performance in terms of both time and test error to the state-of-the-art competitors for most of the $\lambda$ considered.

\end{document}